\documentclass[sigconf]{acmart}
\renewcommand\footnotetextcopyrightpermission[1]{}
\usepackage[utf8]{inputenc} % allow utf-8 input
\usepackage[T1]{fontenc}    % use 8-bit T1 fonts
\usepackage[english]{babel}
\usepackage{hyperref}       % hyperlinks
\usepackage{url}            % simple URL typesetting
\usepackage{booktabs}       % professional-quality tables
\usepackage{amsfonts}       % blackboard math symbols
\usepackage{nicefrac}       % compact symbols for 1/2, etc.
\usepackage{microtype}      % microtypography
\usepackage{lipsum}
\usepackage{fancyhdr}       % header
\usepackage{graphicx}       % graphics
\graphicspath{{media/}}     % organize

\usepackage{amsmath}
\usepackage{mathtools}
\usepackage{amsthm}
\usepackage{algorithm}
\usepackage{algpseudocode}
\usepackage{enumitem}
\usepackage{bm}
\usepackage{xspace}

\usepackage{adjustbox}

\usepackage{graphicx}
\usepackage{subfigure}
\usepackage{subcaption}

\makeatletter
\newtheoremstyle{def}
  {\topsep}
  {\topsep}
  {\normalfont}
  {2pt}
  {\bfseries\MakeUppercase}
  {.}
  {.5em}
  {}
\makeatother

\theoremstyle{definition}

\newtheorem{definition}{Definition}

\AtBeginDocument{%
  }

\setitemize{noitemsep,topsep=3pt,parsep=0em}
\newcommand{\sysname}{\textrm{FIRE}\xspace}
\newcommand{\para}[1]{\noindent\textbf{#1}}

\newtheorem{lemma}{Lemma}[section] %new command for lemma

\begin{document}

%%
%% The "title" command has an optional parameter,
%% allowing the author to define a "short title" to be used in page headers.
\title{A Unified Frequency Domain Decomposition Framework for Interpretable and Robust Time Series Forecasting}

\author{Cheng He}
%\authornote{Both authors contributed equally to this research.}
\email{cheng.he@mail.ustc.edu.cn}
%\orcid{1234-5678-9012}
\affiliation{%
  \institution{University of Science and Technology of China}
  \city{Hefei}
%  \state{Ohio}
  \country{China}
}

\author{Xijie Liang}
\email{lxjie@blackwingasset.com}
\affiliation{%
  \institution{Shanghai Black Wing Asset Management Co., Ltd.}
  \city{Shanghai}
  \country{China}}

\author{Zengrong Zheng}
\email{lxjie@blackwingasset.com}
\affiliation{%
 \institution{Di-Matrix(Shanghai) Information Technology Co., Ltd.}
 \city{Shanghai}
% \state{Arunachal Pradesh}
 \country{China}}

\author{Patrick P.C. Lee}
\email{pclee@cse.cuhk.edu.hk}
\affiliation{%
  \institution{The Chinese University of Hong Kong}
  \city{Hong Kong}
%  \state{Beijing Shi}
  \country{China}}

\author{Xu Huang}
\email{xuhuangcs@mail.ustc.edu.cn}
\affiliation{%
  \institution{University of Science and Technology of China}
  \city{Hefei}
%  \state{Beijing Shi}
  \country{China}}

\author{Zhaoyi Li}
\email{lizhaoyi777@mail.ustc.edu.cn}
\affiliation{%
  \institution{University of Science and Technology of China}
  \city{Hefei}
%  \state{Beijing Shi}
  \country{China}}

\author{Hong Xie}
\email{xiehong2018@foxmail.com}
\affiliation{%
  \institution{University of Science and Technology of China}
  \city{Hefei}
%  \state{Beijing Shi}
  \country{China}}

\author{Defu Lian}
\email{liandefu@ustc.edu.cn}
\affiliation{%
  \institution{University of Science and Technology of China}
  \city{Hefei}
%  \state{Beijing Shi}
  \country{China}}

\author{Enhong Chen}
\email{cheneh@ustc.edu.cn}
\affiliation{%
  \institution{University of Science and Technology of China}
  \city{Hefei}
%  \state{Texas}
  \country{China}
}

\begin{abstract}
Current approaches for time series forecasting, whether in the time or frequency domain, predominantly use deep learning models based on linear layers or transformers. They often encode time series data in a black-box manner and rely on trial-and-error optimization solely based on forecasting performance, leading to limited interpretability and theoretical understanding. Furthermore, the dynamics in data distribution over time and frequency domains pose a critical challenge to accurate forecasting.
%which substantially constrains the potential for achieving robust and consistent predictive gains.
We propose \textbf{\sysname}, a unified frequency domain decomposition framework that provides a mathematical abstraction for diverse types of time series, so as to achieve interpretable and robust time series forecasting. 
%This principled abstraction directly guides the model design, enhancing both interpretability and robustness to distributional drift and basis evolution.
\sysname introduces several key innovations: (i) independent modeling of amplitude and phase components, (ii) adaptive learning of weights of frequency basis components, (iii) a targeted loss function, and (iv) a novel training paradigm for sparse data. 
%a novel composite loss function comprising a hybrid convergence term for joint optimization of point-wise prediction and distributional alignment (via strong and weak convergence), an FFT-based loss to explicitly address basis evolution, and a phase regularization term to ensure stable phase dynamics. 
Extensive experiments demonstrate that \sysname consistently outperforms state-of-the-art models on long-term forecasting benchmarks, achieving superior predictive performance and significantly enhancing interpretability of time series representations.
\end{abstract}

\maketitle

\begin{sloppypar}
\section{Introduction}
\label{sec:intro}

Time series forecasting is a critical yet challenging task in various domains such as web mining, predictive maintenance of IoT devices, traffic prediction, weather forecasting, electricity load management, and financial analysis. Recently, the attention mechanism \cite{Vaswani17} has proven highly effective, establishing transformer-based architectures as the dominant approach for time series representation learning in the temporal domain \cite{Zhou21, Wu21, Li22a, Wang24, Xu22, Song23, Yang23, Du23}.  These models outperform traditional recurrent neural networks (RNNs) and convolutional neural networks (CNNs) \cite{Flunkert17, Lai18, Bai18, Xue19, Cheng20, Durairaj22, Aiwansedo24}, particularly effective in capturing long-range dependencies. However, time series data, composed of temporally ordered scalar sequences, often fail to capture complex underlying patterns when analyzed solely in the temporal domain.

To effectively capture the complex patterns of time series data, recent research has explored frequency-domain representations for time series data using Fast Fourier Transform (FFT). Notable approaches such as FredFormer \cite{Piao24}, FreTS \cite{Yi23}, and FITS \cite{Xu24} leverage frequency-domain techniques, including channel-wise attention, frequency-temporal dependency modeling, and complex-valued interpolation, while WPMixer \cite{Murad25} employs wavelet decomposition combined with multiplayer perceptrons (MLPs) for long-term forecasting.  Hybrid architectures, such as CDX-Net \cite{Li22b}, FEDformer \cite{Zhou22}, and TimeMixer++ \cite{Wangsy24}, integrate temporal and frequency-domain features to enhance the robustness and accuracy of time series representations. Despite these advances, most existing models encode time series representations empirically in a black-box manner, optimized through trial-and-error based on forecasting outcomes. This limits both interpretability and theoretical insights into the underlying data structure.

Time series forecasting is further complicated by concept drift \cite{Tsymbal04, Gama14, Boulegane22, Gunasekara24}, where the statistical properties and patterns of time series data shift over time. Such dynamics also imply \emph{basis evolution} in the frequency domain when time series data is decomposed into frequency components via FFT \cite{Masud10, Gurjar15, Haque16, zhou24}, where new frequency bases appear and existing ones disappear. Basis evolution complicates frequency-domain analysis, as models built on static basis assumptions cannot readily maintain stable and accurate representations.  Consequently, models trained on historical data become less effective for future predictions under concept drift and basis evolution. However, current state-of-the-art forecasting models either overlook these phenomena or address them only implicitly, leaving a void in interpretable and robust solutions.

%In addition, classical machine learning algorithms typically rely on \emph{strong convergence} \cite{Vapnik20} for optimization (Section~\ref{sec:weak_convergence}), which requires stricter pointwise stability and provides stronger theoretical guarantees. However, this often demands large amounts of data and results in slower training. Unlike large language models trained on massive corpora, time series analysis usually lacks large-scale, high-quality open datasets.  Designing optimization frameworks that can efficiently address these challenges remains an open problem.

We propose \textbf{\sysname}, a novel unified \underline{f}requency domain decomposition framework for \underline{i}nterpretable and \underline{r}obust time s\underline{e}ries forecasting. It provides a consistent mathematical abstraction for diverse time series data under concept drift and basis evolution, thereby enabling interpretable frequency-domain representations. It incorporates several key features: (i) modeling amplitude and phase components independently to capture underlying temporal dynamics in concept drift, (ii) learning adaptively the weights of frequency basis components across data patches to track the evolving importance of frequency bases, (iii) a targeted loss function that explicitly accounts for basis evolution, and (iv) a novel training paradigm that integrates Huber loss with a hybrid strong and weak convergence framework to accelerate training and improve generalization, particularly when large-scale, high-quality open datasets are limited. Our main contributions are summarized as follows:
\begin{itemize}[leftmargin=*]
    \item We propose \sysname, a unified frequency domain decomposition framework that provides analytical modeling for diverse types of time series. \sysname incorporates several key techniques to achieve interpretability and robustness for time series forecasting.
    %\item \sysname independently models amplitude and phase, adaptively learns basis weights, and incorporates a novel loss function that explicitly addresses both concept drift and basis evolution, enhancing interpretability and robustness. Additionally, we introduce a training paradigm that combines Huber loss with strong and weak convergence to improve training efficiency and generalization in data-scarce scenarios.
    \item Extensive experiments demonstrate that \sysname consistently outperforms state-of-the-art baselines across various long-term forecasting tasks, delivering cost-effective and interpretable solutions suitable for industrial applications.
\end{itemize}

\section{Preliminaries}
\label{sec:preliminary}

We introduce the analytical formulation of time series data, key concepts of concept drift and basis evolution, and the notations and metrics used in this paper.

\subsection{Analytical Formulation of Time Series Data}
\label{subsec:formulation}

In mathematics and engineering, complex data can be represented as an infinite series of basis vectors. The Fourier series, a widely used set of basis vectors, effectively represents time series data that satisfy specific conditions. Specifically, if a time series $x(t)$ is a periodic signal with period $T$ and satisfies the Dirichlet conditions (i.e., absolutely integrable over one period, with a finite number of discontinuities of the first kind and extrema), then the signal can be accurately represented by a Fourier series. 

We define $X[k]$ as a discrete Fourier transform (DFT) of $x(n)$ as: 
\begin{align} 
X[k] = \sum_{n=0}^{N-1} x[n] \cdot e^{-j\frac{2\pi}{N}kn}, 
\label{eq:DFT}
\end{align}
where $n$ is the time index in the temporal domain, $k$ is the frequency index in the frequency domain, both ranging from $0$ to $N-1$, and $j$ is the imaginary unit satisfying $j^2 = -1$. Using Euler's formula, we can express the exponential term in trigonometric form as: 
\begin{align} 
e^{-j\frac{2\pi}{N}kn} = \cos\left(\frac{2\pi}{N}kn\right) - j\sin\left(\frac{2\pi}{N}kn\right)
\label{eq:Euler expand}
\end{align}

Substituting Equation~\eqref{eq:Euler expand} into Equation~\eqref{eq:DFT}, we can obtain the real part $a[k]$ and the imaginary part $b[k]$ as:
\begin{equation}
\begin{aligned} 
&X[k] = \sum_{n=0}^{N-1} x[n] \cdot \left[\cos\left(\frac{2\pi}{N}kn\right) - j\sin\left(\frac{2\pi}{N}kn\right)\right]
\\
&a[k] = \sum_{n=0}^{N-1} x[n] \cdot \cos\left(\frac{2\pi}{N}kn\right)
\\
&b[k] = -\sum_{n=0}^{N-1} x[n] \cdot \sin\left(\frac{2\pi}{N}kn\right)
\\
&X[k] = a[k] + j \cdot b[k]
\label{eq:Euler DFT}
\end{aligned}
\end{equation}

We can derive the amplitude $A[k]$ and phase $\phi[k]$ in the frequency domain as:
\begin{equation}
\begin{aligned} 
&A[k] = \sqrt{a[k]^2 + b[k]^2}\\
&\phi[k] = \arctan\left(\frac{b[k]}{a[k]}\right).
\label{eq:amp&phi}
\end{aligned}
\end{equation}

The inverse DFT reconstructs the time series $x(t)$ as:
\begin{align} 
x[n] = a_0 + \sum_{k=1}^{N-1} \beta_k A[k] \cdot \cos\left(\frac{2\pi}{N}kn - \phi[k]\right),
\label{eq:reformulation}
\end{align}
where $a_0 = \frac{A[0]}{N}$ is the DC component (intercept) corresponding to $k=0$ in the frequency domain, and $\beta_k$ is the weight of the $k$-th basis component. Although the Fourier series is strictly defined for periodic signals, non-periodic signals can be approximated by assuming the sequence period matches the number of time points. Thus, we can have a uniform representation for $x(t)$ from Equation~\eqref{eq:reformulation}.

\subsection{Concept Drift in Time Domain}
\label{sec:data_drift}

Most machine learning algorithms assume stationary statistical distributions between training and testing phases. However, in practical time series applications, the underlying data distribution often evolves, leading to distinct patterns in future data compared to historical data. This phenomenon, termed \emph{concept drift}, refers to temporal changes in statistical properties \cite{Tsymbal04}. In data stream mining, methods like ADWIN \cite{Bifet07} can be used to detect change points and identify shifts in data concepts over time. 

%To quantify concept drift within a dataset, we leverage ADWIN to detect change points and define the degree of concept drift as follows:

\begin{definition}[\textbf{Degree of concept drift}]
Let $N_{\mathrm{change}}$ be the number of detected change points and $N_{\mathrm{total}}$ be the total number of time points in a dataset. The degree of concept drift $D_{\mathrm{drift}}$ is defined as:
\begin{equation}
    D_{\mathrm{drift}} = \frac{N_{\mathrm{change}}}{N_{\mathrm{total}}}
\end{equation}
\end{definition}

A higher $D_{\mathrm{drift}}$ indicates more frequent changes in the data distribution, and hence greater concept drift.

\subsection{Basis Evolution in Frequency Domain}
\label{sec:basis_evolution}

Concept evolution traditionally refers to the emergence of new classes or concepts in data streams \cite{Masud10}, and can be extended to changes in the underlying basis functions in frequency-domain time series analysis. After time series data is transformed into the frequency domain via FFT, it is segmented into patches, each represented by a vector of $N$ basis energies:
$\mathbf{E}^{(q)} = \big(E_1^{(q)}, E_2^{(q)}, \ldots, E_N^{(q)}\big), \quad q = 1, 2, \ldots, Q$,
where \(E_k^{(q)} \geq 0\) is the energy of the \(k\)-th frequency basis in patch \(q\).

\begin{definition}[\textbf{Basis evolution criterion}]
For each basis $k$, the relative energy change between two consecutive patches $q-1$ and $q$ is:
\begin{equation}
\delta_k^{(q)} = \frac{|E_k^{(q)} - E_k^{(q-1)}|}{E_k^{(q-1)} + \eta},
\end{equation}
where $\eta > 0$ is a small constant to avoid division by zero. Basis $k$ is said to \emph{evolve} at patch $q$ if
\begin{equation}
\delta_k^{(q)} > \epsilon,
\end{equation}
where $\epsilon > 0$ is a fixed threshold.
\end{definition}

\begin{definition}[\textbf{Patch-level basis evolution}]
A patch $q$ is considered to exhibit basis evolution if the fraction of evolving bases exceeds a threshold $\tau \in (0,1]$:
\begin{equation}
\frac{1}{N} \sum_{k=0}^{N-1} \mathbf{1}\big(\delta_k^{(q)} > \epsilon\big) > \tau,
\end{equation}
where $\mathbf{1}(\cdot)$ is the indicator function.
\end{definition}

\begin{definition}[\textbf{Degree of basis evolution}]
Let $\mathcal{Q}_e = \{ q \mid \text{patch } q \text{ exhibits basis evolution} \}$ be the set of evolving patches. The degree of basis evolution over $Q$ patches is:
\begin{equation}
D_{evolution} = \frac{|\mathcal{Q}_e|}{Q} \in [0,1].
\end{equation}
\end{definition}

Basis evolution reflects the non-stationary nature of time series, as the frequency components that characterize the data evolve over time. The non-stationary nature complicates modeling and prediction in the frequency domain.

\subsection{Strong and Weak Convergence}
\label{sec:weak_convergence}

In statistical learning theory~\cite{Vapnik20}, convergence in Hilbert spaces is categorized into \emph{strong} and \emph{weak convergence}~\cite{Li25}. Specifically, a sequence of functions $\{f_h(\boldsymbol{x})\}_{h=1}^{\infty}$ is said to converge strongly to a target function $f(\boldsymbol{x})$ if:
\begin{equation}
    \lim_{h \to \infty} \| f_h(\boldsymbol{x}) - f(\boldsymbol{x}) \| = 0,
    \label{eq:strong_convergence}
\end{equation}
where the norm is defined in the corresponding Hilbert space. In contrast, $\{f_h(\boldsymbol{x})\}_{h=1}^{\infty}$ converges weakly to $f(\boldsymbol{x})$ if:
\begin{equation}
    \lim_{h \to \infty} \langle \phi(\boldsymbol{x}), f_h(\boldsymbol{x}) - f(\boldsymbol{x}) \rangle = 0, \quad \forall\, \phi(\boldsymbol{x}) \in L_2,
    \label{eq:weak_convergence}
\end{equation}
where $\langle \cdot, \cdot \rangle$ denotes the inner product in $L_2$ space.

Strong convergence imposes stricter pointwise stability. It offers robust theoretical guarantees, but requires large datasets and extensive training. 
In contrast, weak convergence focuses on statistical behavior across the data distribution and imposes less restrictive requirements. It enables faster training and better performance with sparse data, while maintaining rigorous mathematical foundations.  In this work, we aim to combine strong and weak convergence. 

%Unlike strong convergence—which requires stricter pointwise stability and provides stronger theoretical guarantees, but often demands large amounts of data and results in slower training—weak convergence focuses on the statistical behavior of functions over the entire data distribution. By imposing less restrictive requirements, weak convergence allows for a broader class of admissible approximations, thereby accelerating training and, through the use of statistical invariants (e.g., predicates), performs well even with sparse data. Its main advantage lies in balancing computational efficiency with adaptability to real-world scenarios, while maintaining rigorous mathematical foundations.

%%%%%%%%%%%%%%%
\section{\sysname Design}

We present \sysname's design, aiming to address concept drift and basis evolution. 
%and provide a comprehensive explanation of its design for time series forecasting, grounded in the analytical formulation presented in Section~\ref{sec:preliminary}. In addition, w

\subsection{Model Architecture}

To effectively capture complex temporal dependencies and concept drift, \sysname primarily operates in the frequency domain. Specifically, the raw multivariate time series data is first preprocessed and transformed into the frequency domain via the Fast Fourier Transform (FFT), which decomposes the signals into orthogonal sinusoidal basis functions. This transformation, along with the resulting frequency domain representation, reveals rich spectral characteristics that facilitate the design of specialized modules capable of modeling intricate correlations and evolving patterns in the data, while adaptively handling basis evolution.
\sysname is composed of three main components, illustrated in Figure~\ref{fig:method}:

\begin{itemize}[leftmargin=*]
    \item \textbf{Embedding and transformation:} This module applies Channel Independent (CI) processing and Instance Normalization (IN) to the raw input data. The normalized data is segmented into patches and converted into the frequency domain via FFT. These frequency-domain patches are then embedded into a high-dimensional feature space through a dedicated embedding layer, enabling effective feature extraction.
    
    \item \textbf{Frequency domain backbone:} Operating on complex-valued frequency patches, this backbone employs complex linear layers to capture intra-patch correlations. It explicitly models amplitude and phase components to handle concept drift, while an attention mechanism adaptively learns weights for the sinusoidal basis to address basis evolution. The processed features are recombined into complex representations for subsequent processing.
    
    \item \textbf{Output projection module:} This module generates frequency-domain predictions by flattening and applying a linear projection. The predicted signals are then transformed back to the time domain through inverse FFT (iFFT), followed by instance denormalization to produce the final forecasts.

    \item \textbf{Composite loss function:} To effectively handle basis evolution and concept drift, \sysname employs a composite loss combining three terms: a Huber loss with hybrid convergence that balances strong and weak convergence for better generalization under sparse and noisy data; an FFT-domain loss that directly minimizes prediction errors in the frequency domain, thus explicitly addressing basis evolution; and a phase regularization term that enforces smooth phase transitions to enhance stability and robustness of the learned representations.
\end{itemize}

Through the integration of these components within a unified frequency-domain framework, \sysname effectively captures both global and local temporal dynamics, enabling interpretable and robust time series forecasting.

\begin{figure}[!t]
\centering
\includegraphics[width=3.3in]{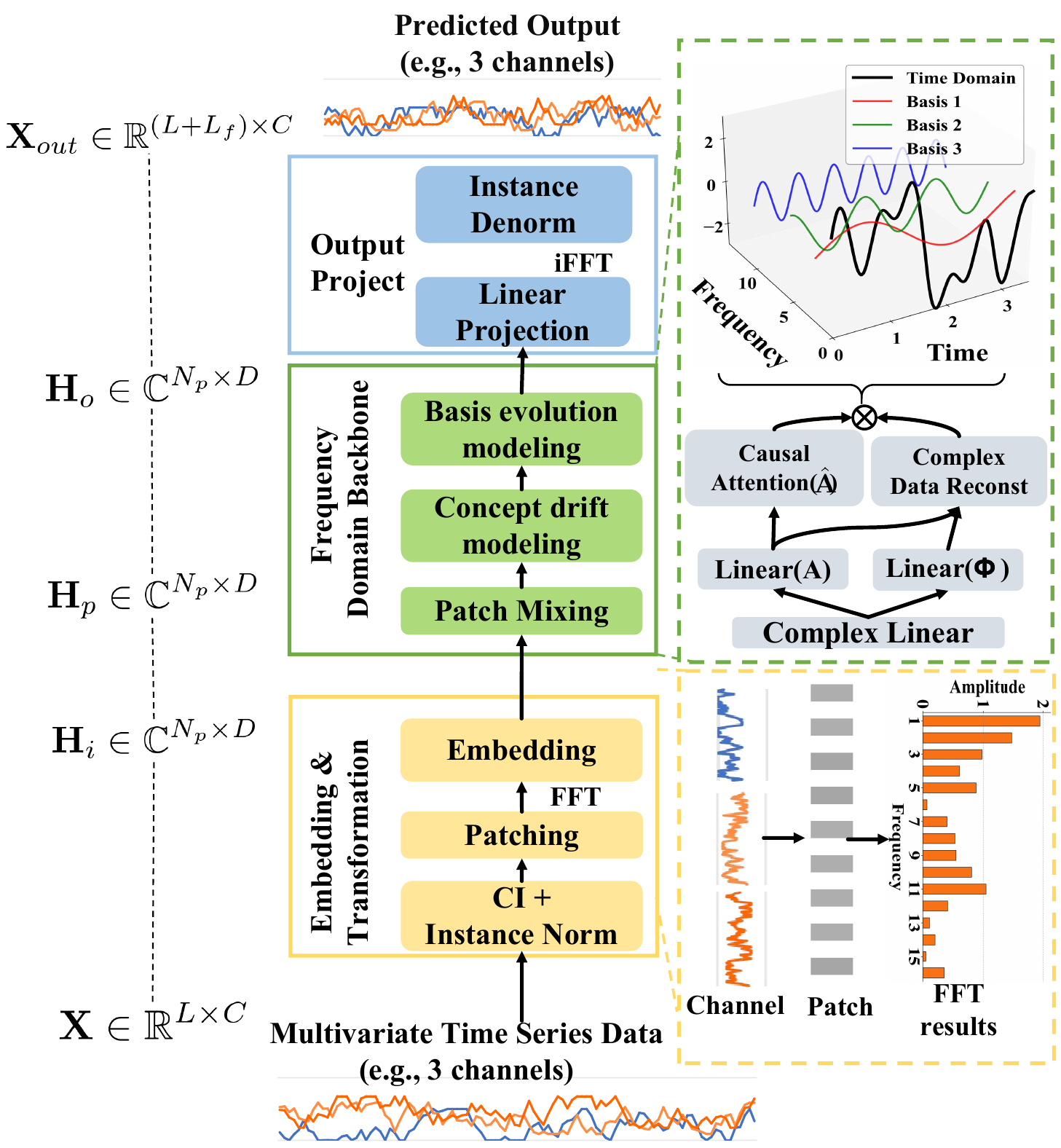} 
\caption{Model architecture of \sysname. It tranforms multivariate time series into the frequency domain through a sequence of steps including CI, IN, patching, and FFT. It captures intra-patch correlations using complex linear layers. It models concept drift via linear transformations, and basis evolution via causal attention mechanisms. It finally generates predictions by a flattened linear projection layer.}
\label{fig:method}
\end{figure}

\subsection{Embedding and Transformation}

Let $\mathbf{X} = [X_{c,l} : c \in [C], l \in [L]]$ denote a multivariate time series instance with $C$ variables and $L$ timestamps. Each instance is first processed using Channel Independent (CI) processing and segmented into overlapping patches following the patching scheme~\citep{Nie23}. The resulting patches are represented as $\mathbf{X}_P \in \mathbb{R}^{N_p \times L_p}$, where $N_p$ is the number of patches and $L_p$ is the length of each patch. These patches are then transformed into the frequency domain via FFT. Subsequently, a linear embedding layer projects the frequency-domain patches into a higher-dimensional feature space, yielding $\mathbf{H}_i \in \mathbb{C}^{N_p \times D}$, where $D$ denotes the embedding dimension:
\begin{equation}
\begin{aligned}
    \mathbf{X}_P &= \mathrm{FFT}(\mathrm{Patching}(\mathrm{CI}(\mathbf{X}))), \\
    \mathbf{H}_i &= \mathbf{W}_{\mathrm{embed}} \cdot \mathbf{X}_P.
\end{aligned}
\end{equation}
Here, $\mathbf{H}_i$ is a complex-valued tensor capturing rich frequency features for downstream processing.

%%%%%%%%%%%%%%%%%%%%%%

%%%%%%%%%%%%%%%%%%%%%%
\subsection{Frequency Domain Backbone}

Starting from the embedded frequency-domain input $\mathbf{H}_i \in \mathbb{C}^{N_p \times D}$, \sysname applies a complex-valued linear transformation to model intra-patch correlations:
\begin{equation}
    \mathbf{H}_P = \mathrm{Linear}_{\mathbb{C}}(\mathbf{H}_i) = \mathbf{W}_{\mathbb{C}} \cdot \mathbf{H}_i + \mathbf{b}_{\mathbb{C}},
\end{equation}
where $\mathbf{W}_{\mathbb{C}} \in \mathbb{C}^{D \times D}$ and $\mathbf{b}_{\mathbb{C}} \in \mathbb{C}^D$ are learnable complex weights and biases, and the output $\mathbf{H}_P \in \mathbb{C}^{N_p \times D}$ retains the same dimensions as the input.

This complex linear transformation effectively models the localized frequency interactions within each patch, enabling the network to extract rich amplitude and phase information that is crucial for representing temporal dynamics. Such frequency-domain representations naturally facilitate the characterization of concept drift, as temporal distributional shifts manifest as variations in these frequency components.

%%%%%%%%%%%%%%%%%%%%%%%%%
\para{Learning of concept drift.}
Concept drift refers to temporal distributional shifts, which can be equivalently characterized in the frequency domain as variations in amplitude and phase distributions across localized frequency patches (Figure~\ref{fig:concept_drift}).

\begin{lemma}[Equivalence of Concept Drift Modeling in Temporal and Frequency Domains]
A non-stationary time series with time-varying distribution exhibits concept drift. Under linear time-invariant signal decomposition, modeling distributional shifts in the temporal domain is equivalent to modeling independent changes in amplitude and phase in the frequency domain.
\end{lemma}

\begin{proof}
Any time series can be decomposed into frequency components via the Fourier transform (see Section~\ref{sec:preliminary}, Equation~\eqref{eq:reformulation}). Since the discrete Fourier transform (DFT) is a linear, invertible mapping—and the fast Fourier transform (FFT) provides an efficient way to compute it—it preserves all information contained in the original time series. Consequently, any temporal changes in the series, such as shifts in mean, variance, or other distributional properties, manifest as corresponding changes in the amplitude and phase of the frequency components. This one-to-one correspondence guarantees that modeling concept drift in the time domain is fully equivalent to modeling it in the frequency domain, without any loss of information.
\end{proof}

%%%%%%%%%%%%%%%%%%%%%%%
Based on the complex linear transformation output $\mathbf{H}_P$, we extract amplitude $\mathbf{A} \in \mathbb{R}^{N_p \times D}$ and phase $\boldsymbol{\phi} \in [-\pi, \pi]^{N_p \times D}$ components. To effectively capture concept drift, \sysname models amplitude and phase variations across patches independently. Specifically, it employs two separate linear layers to learn the inter-patch correlations for amplitude and phase:
\begin{equation}
\begin{aligned}
    \hat{\mathbf{A}} &= \mathrm{Linear}_{Amp}(\mathbf{A}) = \mathbf{W}_{Amp} \mathbf{A} + \mathbf{b}_{Amp}, \\
    \hat{\boldsymbol{\phi}} &= \mathrm{Linear}_{\phi}(\boldsymbol{\phi}) = \mathbf{W}_{\phi} \boldsymbol{\phi} + \mathbf{b}_{\phi},
\end{aligned}
\label{eq:linear_amp&phi}
\end{equation}
where $\mathbf{W}_{Amp}, \mathbf{W}_{\phi} \in \mathbb{R}^{D \times D}$ and $\mathbf{b}_{Amp}, \mathbf{b}_{\phi} \in \mathbb{R}^D$ are learnable parameters. This disentangled design enables interpretable and effective adaptation to non-stationary time series by separately modeling amplitude and phase drift dynamics.

\begin{figure}[!t]
    \centering
    % 第一个子图，宽度设为文本宽度的一半左右，可根据需要微调
    \subfigure[Amplitude variation between consecutive frequency patches]{
        \includegraphics[width=0.225\textwidth]{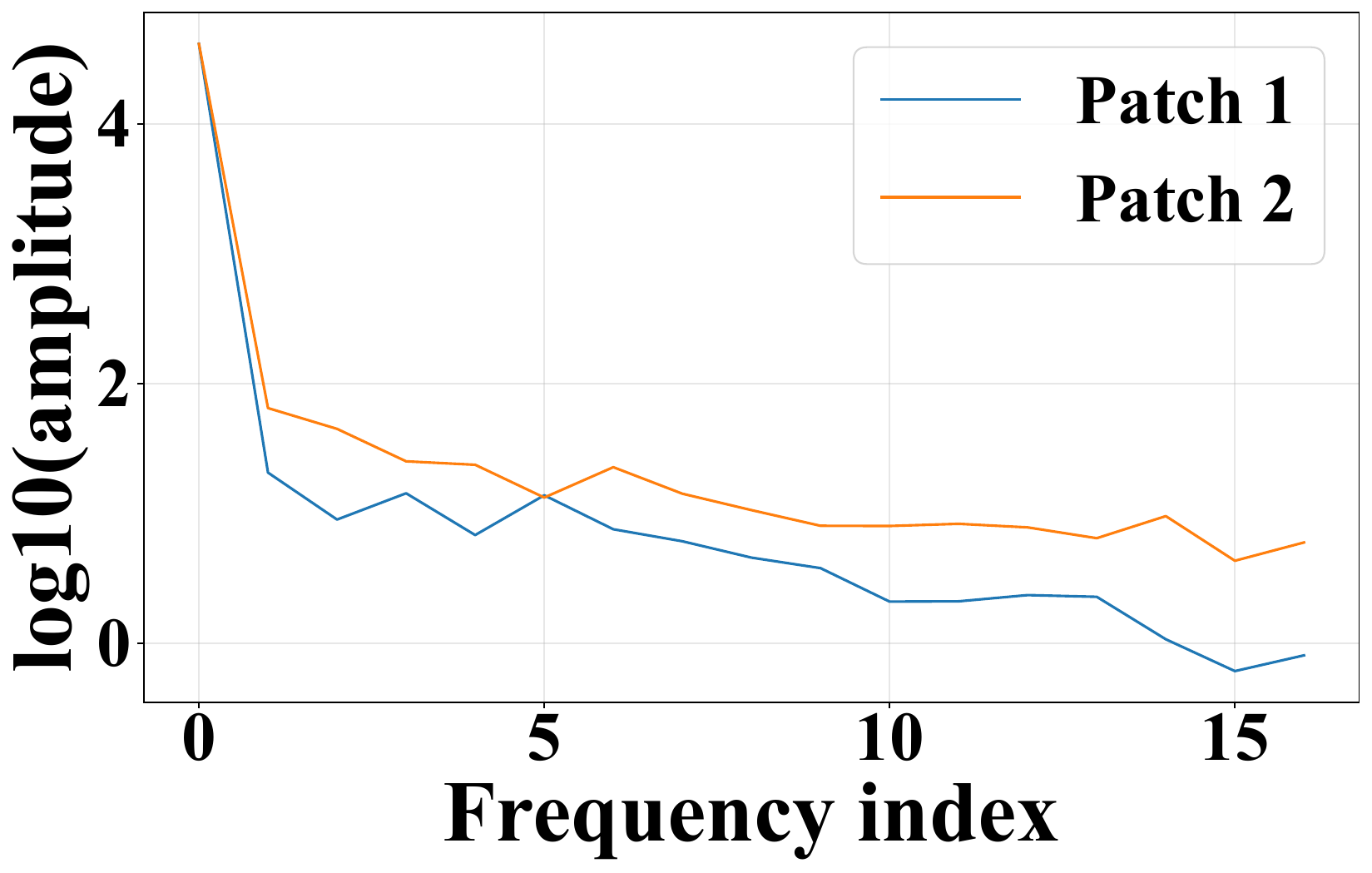}
    }
    % 用于在两个子图之间添加水平间距
    \hfill
    % 第二个子图，宽度同样设为文本宽度的一半左右
    \subfigure[Phase variation between consecutive frequency patches]{
        \includegraphics[width=0.225\textwidth]{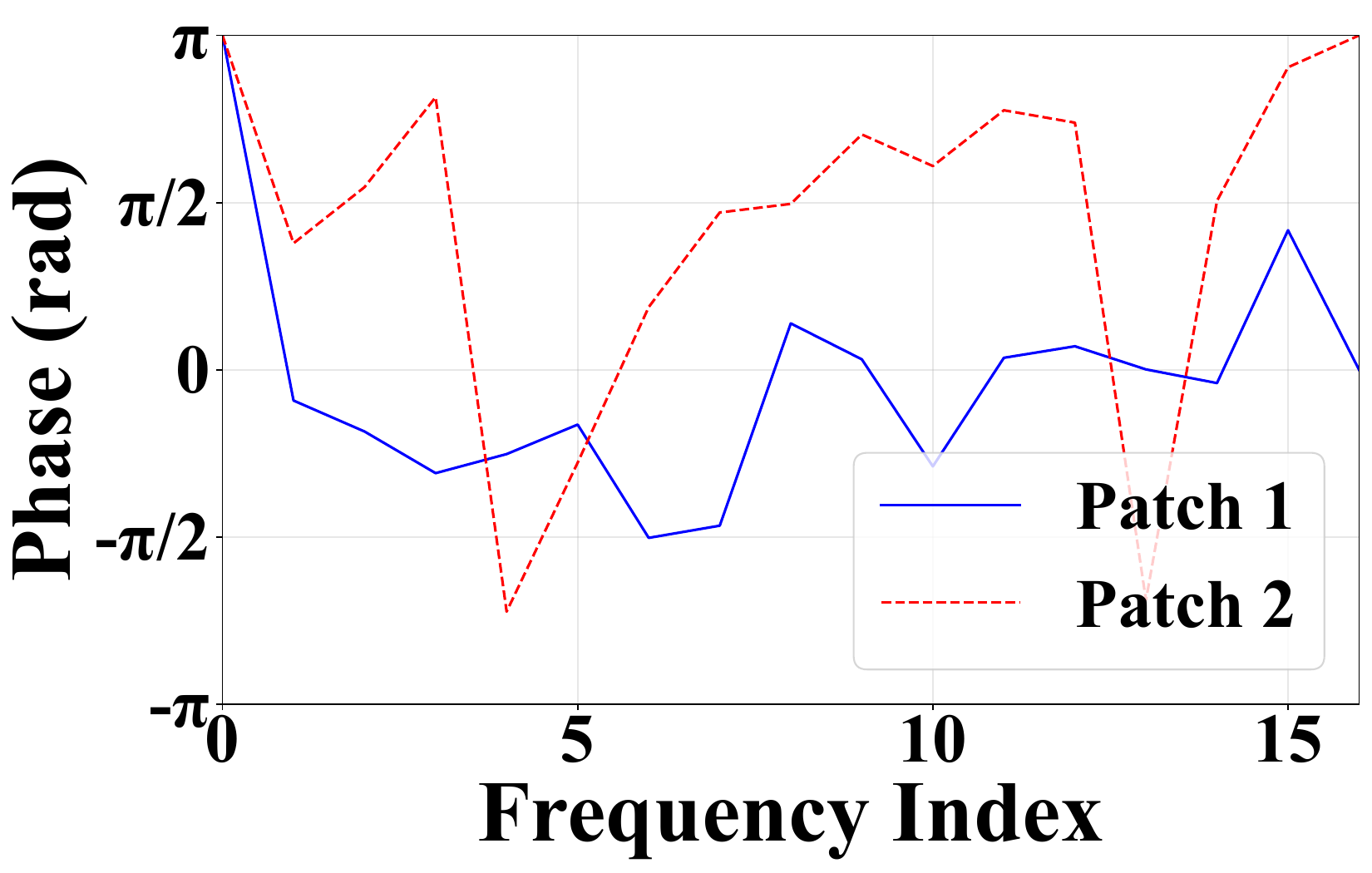}
    }
    \vspace{-9pt}
    \caption{Variations in amplitude and phase distributions between consecutive frequency patches in the frequency domain. The patches are sampled from the Weather and Etth1 datasets, respectively.}
    \label{fig:concept_drift}
    \vspace{-9pt}
\end{figure}

%%%%%%%%%%%%%%%%%%%%%%%%%%%%%
\para{Learning of basis evolution.}
While linear layers effectively capture concept drift through amplitude and phase variations, their capacity to model the more complex, non-linear temporal dynamics of frequency basis evolution is limited. In particular, traditional frequency-domain models relying solely on linear transformations struggle to adapt to abrupt changes or long-range dependencies in the spectral bases. In contrast, causal attention provides a flexible mechanism to dynamically weight and integrate historical amplitude features, making it better suited to handle sudden shifts and intricate basis evolution patterns.

\sysname leverages a \emph{causal masked attention} mechanism applied directly on the previous outputed amplitude representations $\hat{\mathbf{A}} \in \mathbb{R}^{N_p \times D}$ obtained from the amplitude linear layer (Equation~\eqref{eq:linear_amp&phi}). This sequence of amplitude embeddings compactly represents the frequency bases across patches \(p=1,\ldots,N_p\).
The causal attention offers three key advantages over linear layers:

\begin{itemize}[leftmargin=*]
    \item \textbf{Adaptive temporal weighting:} It dynamically learns to weigh historical amplitude features, focusing on the most relevant past patches for the current basis evolution.
    \item \textbf{Modeling long-range dependencies:} Self-attention naturally captures complex dependencies across distant patches, essential for representing gradual or abrupt spectral changes.
    \item \textbf{Preserving causality:} The causal mask ensures that the basis at patch \(p\) depends only on current and past patches \(\leq p\), maintaining temporal consistency required for forecasting.
\end{itemize}

Formally, the amplitude features are projected into queries and keys:
\begin{equation}
\label{eq:att_QK}
\mathbf{Q} = \hat{\mathbf{A}} \mathbf{W}_Q, \quad \mathbf{K} = \hat{\mathbf{A}} \mathbf{W}_K,
\end{equation}
where \(\mathbf{W}_Q, \mathbf{W}_K \in \mathbb{R}^{D \times d}\) are learnable parameters, and \(d\) is the attention dimension.

The scaled dot-product attention scores are masked causally by \(\mathbf{M} \in \{0, -\infty\}^{N_p \times N_p}\):
\[
\mathbf{M}_{p,q} = \begin{cases}
0, & q \leq p \\
-\infty, & q > p
\end{cases},
\]
where $q \leq p$ indicates that attention at patch $p$ is computed only over patch $p$ and all preceding patches $q$, ensuring causality by excluding future patches.
The attention weights $\mathbf{W} \in \mathbb{R}^{N_p \times N_p}$ are computed as
\begin{equation}
\mathbf{W} = \mathrm{softmax}\left(\frac{\mathbf{Q} \mathbf{K}^\top}{\sqrt{d}} + \mathbf{M}\right).
\end{equation}

To further refine intra-patch importance, amplitude vector $\hat{\mathbf{A}}$ is projected by a learnable linear layer:
\begin{equation}
\mathbf{V} = \mathbf{W}_p \hat{\mathbf{A}} + \mathbf{b}_p,
\end{equation}
where $\mathbf{V} \in \mathbb{R}^{N_p \times D}$, $\mathbf{W}_p \in \mathbb{R}^{D \times D}$ and $\mathbf{b}_p \in \mathbb{R}^D$.
The final dynamic weights $\mathbf{U}$ modulating the frequency bases are obtained by combining inter-patch attention and intra-patch projections:
\begin{equation}
\mathbf{U} = \mathbf{W} \mathbf{V},
\end{equation}
with $\mathbf{U} \in \mathbb{R}^{N_p \times D}$.
These adaptive weights $\mathbf{U}$ are applied element-wise to the original frequency bases $\mathbf{B}$, producing the dynamically evolved bases:
\begin{equation}
\label{eq:att_final}
\mathbf{H}_{o} = \mathbf{U} \odot \mathbf{B},
\end{equation}
where $\mathbf{H}_{o} \in \mathbb{C}^{N_p \times D}$, $\odot$ denotes element-wise multiplication.

This causal attention-based design enables \sysname to flexibly and effectively capture complex, non-linear, and temporally adaptive basis evolution patterns, surpassing the representational limitations of traditional linear layers.

%%%%%%%%%%%%%%%%%%%%%%%%%%%%
\subsection{Output Projection}

After backbone processing, \sysname flattens the output $\mathbf{H}_{o}$ and passes it through a linear projection layer to produce predictions in the frequency domain. These are then transformed back to the time domain using iFFT, followed by instance denormalization, to yield the final forecasts:
\begin{equation}
\label{eq:model_output}
\mathbf{X}_{out} = \mathbf{Denorm}(\mathbf{iFFT}(\mathbf{W}_{\text{LinProj}} \cdot \mathbf{Flatten}(\mathbf{H}_{o})))
\end{equation}
where $\mathbf{X_{out}} \in \mathbb{R}^ {L_{pred} \times C}$, with $L_{pred}$ denoting the prediction length.

%%%%%%%%%%%%%%%%%%%%%%%%%%
\subsection{Loss Function}

After the output projection module, we need to quantify the loss between $\mathbf{X}_{out}$ and the ground truth $\mathbf{X}_{true}$.
\sysname employs a composite loss comprising the Huber loss with hybrid convergence ($\mathcal{L}_{wh}$), FFT loss ($\mathcal{L}_{\text{fft}}$), and phase regularization ($\mathcal{R}_{\phi}$). This loss also explicitly guides the model to address concept drift and basis evolution in the frequency domain, thereby providing a clear objective for parameter optimization:
\begin{equation}
\mathcal{L} = \mathcal{L}_{wh} + \mathcal{L}_{\text{fft}} + \mathcal{R}_{\phi}.
\label{eq:loss}
\end{equation}
The individual components are detailed as follows.

\para{Huber loss with hybrid convergence.} To balance strong and weak convergence and improve generalization under sparse data (Section~\ref{sec:preliminary}), \sysname adopts Huber loss~\cite{Huber81}, which smoothly interpolates between $\ell_2$ and $\ell_1$ losses:
\begin{equation}
    \mathcal{L}_{\delta}(\mathbf{X}_{true}, \mathbf{X}_{out}) = \frac{1}{L_{pred}}\sum_{l=1}^{L_{pred}}\delta^2 \left( \sqrt{1 + \left(\frac{x_{true} - x_{out}}{\delta}\right)^2} - 1 \right)
    \label{eq:Huber_loss}
\end{equation}
where $\delta$ is a hyperparameter controlling the transition threshold. 

To incorporate weak convergence (Equation~\eqref{eq:weak_convergence}), the Huber loss is weighted by a matrix $\mathbf{W} \in \mathbb{R}^{1 \times B}$ (with batch size $B$) that linearly combines identity and predicate-based components:
\begin{equation}
    \mathcal{L}_{wh} = \sum_{b=1}^B \mathbf{W} \cdot \mathcal{L}_{\delta}(x_{true}, x_{out}), \quad
    \mathbf{W} = \hat{\tau} \mathbf{I} + \tau \mathbf{P},
    \label{eq:weak_Huber}
\end{equation}
where $\tau = 1 - \hat{\tau}$ balances strong and weak convergence, $\mathbf{I}$ is the identity matrix, and $\mathbf{P}$ is the empirical covariance matrix of predicates:
\begin{equation}
    \mathbf{P} = \frac{1}{m} \sum_{s=1}^m \psi_s \psi_s^\top.
\end{equation}
where $m$ is the number of predicates. For simplicity, we use a single predicate $\psi(\boldsymbol{x}) = 1$ in this work. This formulation leverages statistical invariants captured by weak convergence to enhance robustness and generalization, particularly in noisy or sparse scenarios.

\para{FFT loss.} The FFT loss, $\mathcal{L}_{\text{fft}}$, is defined as the mean absolute error (MAE) between the predicted and ground truth sequences in the frequency domain:
\begin{equation}
\mathcal{L}_{\text{fft}} = \frac{1}{N_f} \sum_{k=1}^{N_f} \left| \mathbf{FFT}(\mathbf{X}_{true}) - \mathbf{FFT}(\mathbf{X}_{out}) \right|
\label{eq:fft_loss}
\end{equation}
where $N_f$ is the number of bases of the predicted sequence in the frequency domain. This loss explicitly addresses basis evolution by minimizing discrepancies in frequency basis vectors.

\para{Phase regularization.} To ensure smooth and stable phase transitions, \sysname introduces phase regularization to constrain phase changes in the predicted sequence. It formulates a first-order difference penalty:
\begin{equation}
\mathcal{R}_{\phi} = \lambda \frac{1}{D - 1} \sum_{d=1}^{D - 1} \left( \phi_{out}^{d+1} - \phi_{out}^{d} \right)^2, 
\label{eq:phase_reg}
\end{equation}
where $\lambda$ is a weighting factor, $D$ is the model dimensionality, and $\phi_{\text{out}}^d$ denotes the phase feature of the $d$-th dimension. This enhances model robustness and generalizability.

\subsection{Discussion}

Time series forecasting in the time domain is challenging due to complex patterns and limited information. Instead, \sysname first transforms the data into the frequency domain via FFT, which decomposes the signal into multiple frequency basis components. We choose FFT over other basis decomposition methods because it is reversible and parameter-free, requiring no hyperparameter tuning or prior knowledge, thus making it broadly applicable to various time series (see Appendix~\ref{appendix: comparison} for details).

Traditional methods typically model the real and imaginary parts of the transformed signal. However, these lack clear physical interpretation and make it hard to hard to connect the results back to the original data. In contrast, \sysname converts each complex component into amplitude (indicating the strength or energy of each basis) and phase (indicating the timing), modeling them separately (Equation~\eqref{eq:amp&phi}). This decomposition enables the model to capture distinct physical features and concept drift patterns while maintaining a direct link to the original signal.
To better capture the temporal evolution of frequency bases, we introduce a causal attention mechanism that adaptively learns how basis components change and interact across patches (Equations~\eqref{eq:att_QK}--\eqref{eq:att_final}). After forecasting in the frequency domain, the model converts the results back to the time domain. Finally, a composite loss function (Equation~\eqref{eq:loss}) is employed, measuring loss in both time (Equation~\eqref{eq:weak_convergence}) and frequency domains (Equation~\eqref{eq:fft_loss}), while constraining phase shifts (Equation~\eqref{eq:phase_reg}) to ensure smooth and robust predictions.

In summary, \sysname succeeds by extracting more interpretable and physically meaningful features, explicitly modeling their dynamics and interactions, and optimizing with mathematically and physically grounded objectives. This design makes \sysname robust, adaptive, and accurate across diverse real-world forecasting tasks.

\section{Experiments}
\label{sec:exp}

We extensively evaluate the performance of \sysname across a variety of long-term forecasting tasks. We compare \sysname against state-of-the-art baselines, particularly those that emphasize frequency-domain modeling of time series data. We also perform comprehensive ablation studies, hyperparameter sensitivity analyses, and targeted experiments on handling concept drift and basis evolution.

\subsection{Datasets and Baselines}

We conduct experiments on seven widely used public time series forecasting datasets \cite{Wu21} (see Table~\ref{table:datasets}), including the Electricity Transformer Temperature datasets at both hourly and minute-level granularities (ETTh1, ETTh2, ETTm1, ETTm2), as well as Weather, Traffic, and Electricity Power Consumption (ELC). 

%\input{FF/tables/stat_datasets}
%%%%%%%%%%%%%%%%%%%%
% statistics of datasets
%%%%%%%%%%%%%%%%%%%%
\begin{table}[!t]
\centering
\caption{Statistics of datasets}
\label{table:datasets}
\vspace{-9pt}
\begin{tabular}{c|c|c|c}
\toprule
Dataset & Length & Dimension & Frequency \\ \midrule
ETTh & 17420 & 7 & 1 hour \\
ETTm & 69680 & 7 & 15 min \\
Weather & 52696 & 21 & 10 min \\
Electricity & 26304 & 321 & 1 hour \\
Traffic & 17544 & 862 & 1 hour  \\ 
            \bottomrule
\end{tabular}
\end{table}

We select representative baselines for comparison. We reproduce the results of two frequency-based models, FredFormer \cite{Piao24} and WPMixer \cite{Murad25}. For other baselines, including TimeMixer \cite{Wang24}, iTransformer \cite{LiuH24}, PatchTST \cite{Nie23}, and TimesNet \cite{Wu23}, we report the results as published in their respective papers.

\subsection{Experimental Settings}

We choose a look-back window of 96 and forecast future time points $T\in \{96, 192, 336, 720\}$. We use the mean squared error (MSE) and mean absolute error (MAE) as the evaluation metrics and compare the results with the best-performing results of SOTA models presented in papers or reproduced from their published source codes. We implement \sysname in PyTorch \cite{Paszke19} and train it on a single NVIDIA A100 40GB GPU.

%%%%%%%%%%%%%%%%
\subsection{Forecasting Results}

Table \ref{table:forecasting} summarizes the full forecasting results, with the best performance highlighted in \textbf{bold}. 
%We reproduce two SOTA frequency-domain models, Fredformer and WPMixer, and include results from other forecasting methods under comparable experimental settings.
The results show that \sysname consistently outperforms all competitors, achieving the best results in 21 out of 35 tasks based on MSE and 26 out of 35 based on MAE. 
%Considering the average performance across all datasets, \sysname consistently outperforms all competitors, achieving the best results in 21 out of 35 tasks based on MSE and 26 out of 35 based on MAE. 
On average, \sysname improves MSE by 3\%-8\% compared to the second-best model, WPMixer, and by 20\%-30\% compared to the worst-performing model, TimesNet, with the largest gains observed in certain datasets such as ETTh1 and Traffic. Similarly, for MAE, \sysname's relative improvements are 2\%-7\% over WPMixer and 15\%-25\% over TimesNet across various tasks. Our results demonstrate \sysname’s robustness and superior ability to capture complex temporal dynamics for long-term forecasting.

\begin{table*}[!t]
\centering
\caption{Long-term forecasting results for prediction lengths $T\in \{96, 192, 336, 720\}$. Best results are highlighted in bold.}
\label{table:forecasting}
\vspace{-9pt}
\small % 调整字体大小 small/scriptsize/footnotesize
\setlength{\tabcolsep}{4pt} % 调整列间距
\renewcommand{\arraystretch}{0.8} % 调整行间距
\begin{tabular}{c|c|cc|cc|cc|cc|cc|cc|cc}
\hline
Model &  & \multicolumn{2}{c|}{\sysname} & \multicolumn{2}{c|}{Fredformer} & \multicolumn{2}{c|}{WPMixer} & \multicolumn{2}{c|}{TimeMixer} & \multicolumn{2}{c|}{iTransformer} & \multicolumn{2}{c|}{PatchTST} & \multicolumn{2}{c}{TimesNet} \\ \hline
Dataset & T & \multicolumn{1}{c|}{MSE} & MAE & \multicolumn{1}{c|}{MSE} & MAE & \multicolumn{1}{c|}{MSE} & MAE & \multicolumn{1}{c|}{MSE} & MAE & \multicolumn{1}{c|}{MSE} & MAE & \multicolumn{1}{c|}{MSE} & MAE & \multicolumn{1}{c|}{MSE} & MAE \\ \hline
 & 96 & \multicolumn{1}{c|}{\textbf{0.365}} & \textbf{0.390} & \multicolumn{1}{c|}{0.373} & 0.392 & \multicolumn{1}{c|}{0.375} & 0.393 & \multicolumn{1}{c|}{0.375} & 0.400 & \multicolumn{1}{c|}{0.386} & 0.405 & \multicolumn{1}{c|}{0.460} & 0.447 & \multicolumn{1}{c|}{0.384} & 0.402 \\
 & 192 & \multicolumn{1}{c|}{\textbf{0.420}} & 0.418 & \multicolumn{1}{c|}{0.433} & 0.420 & \multicolumn{1}{c|}{0.428} & \textbf{0.417} & \multicolumn{1}{c|}{0.429} & 0.421 & \multicolumn{1}{c|}{0.441} & 0.436 & \multicolumn{1}{c|}{0.512} & 0.477 & \multicolumn{1}{c|}{0.436} & 0.429 \\
ETTh1 & 336 & \multicolumn{1}{c|}{\textbf{0.458}} & \textbf{0.437} & \multicolumn{1}{c|}{0.470} & 0.437 & \multicolumn{1}{c|}{0.477} & \textbf{0.439} & \multicolumn{1}{c|}{0.484} & 0.458 & \multicolumn{1}{c|}{0.487} & 0.458 & \multicolumn{1}{c|}{0.546} & 0.496 & \multicolumn{1}{c|}{0.638} & 0.469 \\
 & 720 & \multicolumn{1}{c|}{\textbf{0.456}} & \textbf{0.454} & \multicolumn{1}{c|}{0.467} & 0.456 & \multicolumn{1}{c|}{0.460} & 0.454 & \multicolumn{1}{c|}{0.498} & 0.482 & \multicolumn{1}{c|}{0.503} & 0.491 & \multicolumn{1}{c|}{0.544} & 0.517 & \multicolumn{1}{c|}{0.521} & 0.500 \\
 & Avg. & \multicolumn{1}{c|}{\textbf{0.425}} & \textbf{0.425} & \multicolumn{1}{c|}{0.436} & 0.426 & \multicolumn{1}{c|}{0.435} & 0.426 & \multicolumn{1}{c|}{0.447} & 0.440 & \multicolumn{1}{c|}{0.454} & 0.447 & \multicolumn{1}{c|}{0.516} & 0.484 & \multicolumn{1}{c|}{0.495} & 0.450 \\ \hline
 & 96 & \multicolumn{1}{c|}{\textbf{0.282}} & \textbf{0.333} & \multicolumn{1}{c|}{0.293} & 0.342 & \multicolumn{1}{c|}{0.283} & 0.335 & \multicolumn{1}{c|}{0.289} & 0.341 & \multicolumn{1}{c|}{0.297} & 0.349 & \multicolumn{1}{c|}{0.308} & 0.355 & \multicolumn{1}{c|}{0.340} & 0.374 \\
 & 192 & \multicolumn{1}{c|}{\textbf{0.362}} & \textbf{0.383} & \multicolumn{1}{c|}{0.371} & 0.389 & \multicolumn{1}{c|}{0.364} & 0.391 & \multicolumn{1}{c|}{0.372} & 0.392 & \multicolumn{1}{c|}{0.380} & 0.400 & \multicolumn{1}{c|}{0.393} & 0.405 & \multicolumn{1}{c|}{0.402} & 0.414 \\
ETTh2 & 336 & \multicolumn{1}{c|}{0.403} & 0.419 & \multicolumn{1}{c|}{\textbf{0.382}} & \textbf{0.409} & \multicolumn{1}{c|}{0.409} & 0.424 & \multicolumn{1}{c|}{0.386} & 0.414 & \multicolumn{1}{c|}{0.428} & 0.432 & \multicolumn{1}{c|}{0.427} & 0.436 & \multicolumn{1}{c|}{0.452} & 0.452 \\
 & 720 & \multicolumn{1}{c|}{\textbf{0.408}} & \textbf{0.433} & \multicolumn{1}{c|}{0.415} & 0.434 & \multicolumn{1}{c|}{0.429} & 0.443 & \multicolumn{1}{c|}{0.412} & 0.434 & \multicolumn{1}{c|}{0.427} & 0.445 & \multicolumn{1}{c|}{0.436} & 0.450 & \multicolumn{1}{c|}{0.462} & 0.468 \\
 & Avg. & \multicolumn{1}{c|}{\textbf{0.364}} & \textbf{0.392} & \multicolumn{1}{c|}{0.365} & 0.394 & \multicolumn{1}{c|}{0.371} & 0.398 & \multicolumn{1}{c|}{0.364} & 0.395 & \multicolumn{1}{c|}{0.383} & 0.407 & \multicolumn{1}{c|}{0.391} & 0.411 & \multicolumn{1}{c|}{0.414} & 0.427 \\ \hline
 & 96 & \multicolumn{1}{c|}{\textbf{0.310}} & \textbf{0.344} & \multicolumn{1}{c|}{0.326} & 0.361 & \multicolumn{1}{c|}{0.316} & 0.352 & \multicolumn{1}{c|}{0.320} & 0.357 & \multicolumn{1}{c|}{0.334} & 0.368 & \multicolumn{1}{c|}{0.352} & 0.374 & \multicolumn{1}{c|}{0.338} & 0.375 \\
 & 192 & \multicolumn{1}{c|}{\textbf{0.356}} & \textbf{0.375} & \multicolumn{1}{c|}{0.363} & 0.380 & \multicolumn{1}{c|}{0.362} & 0.376 & \multicolumn{1}{c|}{0.361} & 0.381 & \multicolumn{1}{c|}{0.377} & 0.391 & \multicolumn{1}{c|}{0.390} & 0.393 & \multicolumn{1}{c|}{0.374} & 0.387 \\
ETTm1 & 336 & \multicolumn{1}{c|}{\textbf{0.385}} & 0.397 & \multicolumn{1}{c|}{0.395} & 0.403 & \multicolumn{1}{c|}{0.387} & \textbf{0.396} & \multicolumn{1}{c|}{0.390} & 0.404 & \multicolumn{1}{c|}{0.426} & 0.420 & \multicolumn{1}{c|}{0.421} & 0.414 & \multicolumn{1}{c|}{0.410} & 0.411 \\
 & 720 & \multicolumn{1}{c|}{0.448} & \textbf{0.431} & \multicolumn{1}{c|}{0.453} & 0.438 & \multicolumn{1}{c|}{\textbf{0.447}} & 0.432 & \multicolumn{1}{c|}{0.454} & 0.441 & \multicolumn{1}{c|}{0.491} & 0.459 & \multicolumn{1}{c|}{0.462} & 0.449 & \multicolumn{1}{c|}{0.478} & 0.450 \\
 & Avg. & \multicolumn{1}{c|}{\textbf{0.375}} & \textbf{0.387} & \multicolumn{1}{c|}{0.384} & 0.396 & \multicolumn{1}{c|}{0.378} & 0.389 & \multicolumn{1}{c|}{0.381} & 0.395 & \multicolumn{1}{c|}{0.407} & 0.410 & \multicolumn{1}{c|}{0.406} & 0.407 & \multicolumn{1}{c|}{0.400} & 0.406 \\ \hline
 & 96 & \multicolumn{1}{c|}{\textbf{0.170}} & \textbf{0.252} & \multicolumn{1}{c|}{0.177} & 0.259 & \multicolumn{1}{c|}{0.171} & 0.252 & \multicolumn{1}{c|}{0.175} & 0.258 & \multicolumn{1}{c|}{0.180} & 0.264 & \multicolumn{1}{c|}{0.183} & 0.270 & \multicolumn{1}{c|}{0.187} & 0.267 \\
 & 192 & \multicolumn{1}{c|}{0.237} & 0.297 & \multicolumn{1}{c|}{0.243} & 0.301 & \multicolumn{1}{c|}{\textbf{0.233}} & \textbf{0.294} & \multicolumn{1}{c|}{0.237} & 0.299 & \multicolumn{1}{c|}{0.250} & 0.309 & \multicolumn{1}{c|}{0.255} & 0.314 & \multicolumn{1}{c|}{0.249} & 0.309 \\
ETTm2 & 336 & \multicolumn{1}{c|}{0.299} & 0.338 & \multicolumn{1}{c|}{0.302} & 0.340 & \multicolumn{1}{c|}{\textbf{0.290}} & \textbf{0.333} & \multicolumn{1}{c|}{0.298} & 0.340 & \multicolumn{1}{c|}{0.311} & 0.348 & \multicolumn{1}{c|}{0.309} & 0.347 & \multicolumn{1}{c|}{0.321} & 0.351 \\
 & 720 & \multicolumn{1}{c|}{0.399} & 0.395 & \multicolumn{1}{c|}{0.397} & 0.396 & \multicolumn{1}{c|}{\textbf{0.387}} & \textbf{0.390} & \multicolumn{1}{c|}{0.391} & 0.396 & \multicolumn{1}{c|}{0.412} & 0.407 & \multicolumn{1}{c|}{0.412} & 0.404 & \multicolumn{1}{c|}{0.408} & 0.403 \\
 & Avg. & \multicolumn{1}{c|}{0.276} & 0.321 & \multicolumn{1}{c|}{0.280} & 0.324 & \multicolumn{1}{c|}{\textbf{0.270}} & \textbf{0.317} & \multicolumn{1}{c|}{0.275} & 0.323 & \multicolumn{1}{c|}{0.288} & 0.332 & \multicolumn{1}{c|}{0.290} & 0.334 & \multicolumn{1}{c|}{0.291} & 0.333 \\ \hline
 & 96 & \multicolumn{1}{c|}{\textbf{0.162}} & \textbf{0.204} & \multicolumn{1}{c|}{0.163} & 0.207 & \multicolumn{1}{c|}{0.162} & 0.204 & \multicolumn{1}{c|}{0.163} & 0.209 & \multicolumn{1}{c|}{0.174} & 0.214 & \multicolumn{1}{c|}{0.186} & 0.227 & \multicolumn{1}{c|}{0.172} & 0.220 \\
 & 192 & \multicolumn{1}{c|}{\textbf{0.207}} & 0.246 & \multicolumn{1}{c|}{0.211} & 0.251 & \multicolumn{1}{c|}{0.209} & \textbf{0.246} & \multicolumn{1}{c|}{0.208} & 0.250 & \multicolumn{1}{c|}{0.221} & 0.254 & \multicolumn{1}{c|}{0.234} & 0.265 & \multicolumn{1}{c|}{0.219} & 0.261 \\
Weather & 336 & \multicolumn{1}{c|}{\textbf{0.263}} & 0.287 & \multicolumn{1}{c|}{0.267} & 0.292 & \multicolumn{1}{c|}{0.263} & \textbf{0.287} & \multicolumn{1}{c|}{0.251} & 0.287 & \multicolumn{1}{c|}{0.278} & 0.296 & \multicolumn{1}{c|}{0.284} & 0.301 & \multicolumn{1}{c|}{0.246} & 0.337 \\
 & 720 & \multicolumn{1}{c|}{0.340} & \textbf{0.338} & \multicolumn{1}{c|}{0.343} & 0.341 & \multicolumn{1}{c|}{\textbf{0.340}} & 0.339 & \multicolumn{1}{c|}{0.339} & 0.341 & \multicolumn{1}{c|}{0.358} & 0.347 & \multicolumn{1}{c|}{0.356} & 0.349 & \multicolumn{1}{c|}{0.365} & 0.359 \\
 & Avg. & \multicolumn{1}{c|}{\textbf{0.243}} & \textbf{0.269} & \multicolumn{1}{c|}{0.246} & 0.273 & \multicolumn{1}{c|}{0.244} & 0.269 & \multicolumn{1}{c|}{0.240} & 0.271 & \multicolumn{1}{c|}{0.258} & 0.278 & \multicolumn{1}{c|}{0.265} & 0.285 & \multicolumn{1}{c|}{0.251} & 0.294 \\ \hline
 & 96 & \multicolumn{1}{c|}{0.474} & \textbf{0.272} & \multicolumn{1}{c|}{\textbf{0.406}} & 0.277 & \multicolumn{1}{c|}{0.465} & 0.286 & \multicolumn{1}{c|}{0.462} & 0.285 & \multicolumn{1}{c|}{0.395} & 0.268 & \multicolumn{1}{c|}{0.526} & 0.347 & \multicolumn{1}{c|}{0.593} & 0.321 \\
 & 192 & \multicolumn{1}{c|}{0.487} & \textbf{0.269} & \multicolumn{1}{c|}{\textbf{0.426}} & 0.290 & \multicolumn{1}{c|}{0.475} & 0.290 & \multicolumn{1}{c|}{0.473} & 0.296 & \multicolumn{1}{c|}{0.417} & 0.276 & \multicolumn{1}{c|}{0.522} & 0.332 & \multicolumn{1}{c|}{0.617} & 0.336 \\
Traffic & 336 & \multicolumn{1}{c|}{0.484} & \textbf{0.275} & \multicolumn{1}{c|}{\textbf{0.432}} & 0.281 & \multicolumn{1}{c|}{0.489} & 0.296 & \multicolumn{1}{c|}{0.498} & 0.296 & \multicolumn{1}{c|}{0.433} & 0.283 & \multicolumn{1}{c|}{0.517} & 0.334 & \multicolumn{1}{c|}{0.629} & 0.336 \\
 & 720 & \multicolumn{1}{c|}{0.531} & \textbf{0.295} & \multicolumn{1}{c|}{\textbf{0.463}} & 0.300 & \multicolumn{1}{c|}{0.527} & 0.318 & \multicolumn{1}{c|}{0.506} & 0.313 & \multicolumn{1}{c|}{0.467} & 0.302 & \multicolumn{1}{c|}{0.552} & 0.352 & \multicolumn{1}{c|}{0.64} & 0.35 \\
 & Avg. & \multicolumn{1}{c|}{0.494} & \textbf{0.278} & \multicolumn{1}{c|}{\textbf{0.432}} & 0.287 & \multicolumn{1}{c|}{0.489} & 0.298 & \multicolumn{1}{c|}{0.484} & 0.297 & \multicolumn{1}{c|}{0.428} & 0.282 & \multicolumn{1}{c|}{0.529} & 0.341 & \multicolumn{1}{c|}{0.62} & 0.336 \\ \hline
 & 96 & \multicolumn{1}{c|}{0.148} & \textbf{0.236} & \multicolumn{1}{c|}{\textbf{0.147}} & 0.241 & \multicolumn{1}{c|}{0.150} & 0.241 & \multicolumn{1}{c|}{0.153} & 0.247 & \multicolumn{1}{c|}{0.148} & 0.240 & \multicolumn{1}{c|}{0.190} & 0.296 & \multicolumn{1}{c|}{0.168} & 0.272 \\
 & 192 & \multicolumn{1}{c|}{\textbf{0.161}} & \textbf{0.249} & \multicolumn{1}{c|}{0.165} & 0.258 & \multicolumn{1}{c|}{0.162} & 0.252 & \multicolumn{1}{c|}{0.166} & 0.256 & \multicolumn{1}{c|}{0.162} & 0.253 & \multicolumn{1}{c|}{0.199} & 0.304 & \multicolumn{1}{c|}{0.184} & 0.322 \\
Elc & 336 & \multicolumn{1}{c|}{\textbf{0.176}} & \textbf{0.265} & \multicolumn{1}{c|}{0.177} & 0.273 & \multicolumn{1}{c|}{0.179} & 0.270 & \multicolumn{1}{c|}{0.185} & 0.277 & \multicolumn{1}{c|}{0.178} & 0.269 & \multicolumn{1}{c|}{0.217} & 0.319 & \multicolumn{1}{c|}{0.198} & 0.300 \\
 & 720 & \multicolumn{1}{c|}{0.215} & \textbf{0.299} & \multicolumn{1}{c|}{\textbf{0.213}} & 0.304 & \multicolumn{1}{c|}{0.217} & 0.304 & \multicolumn{1}{c|}{0.225} & 0.310 & \multicolumn{1}{c|}{0.225} & 0.317 & \multicolumn{1}{c|}{0.258} & 0.352 & \multicolumn{1}{c|}{0.220} & 0.320 \\
 & Avg. & \multicolumn{1}{c|}{\textbf{0.175}} & \textbf{0.262} & \multicolumn{1}{c|}{0.176} & 0.269 & \multicolumn{1}{c|}{0.177} & 0.267 & \multicolumn{1}{c|}{0.182} & 0.272 & \multicolumn{1}{c|}{0.178} & 0.270 & \multicolumn{1}{c|}{0.216} & 0.318 & \multicolumn{1}{c|}{0.193} & 0.304 \\ \hline
Best\_count &  & \multicolumn{1}{c|}{\textbf{21/35}} & \textbf{26/35} & \multicolumn{1}{c|}{8} & 1 & \multicolumn{1}{c|}{6} & 9 & \multicolumn{1}{c|}{0} & 0 & \multicolumn{1}{c|}{0} & 0 & \multicolumn{1}{c|}{0} & 0 & \multicolumn{1}{c|}{0} & 0 \\ \hline
\end{tabular}
\end{table*}

\subsection{Ablation Results}

%\para{Module ablation}. 
To comprehensively assess the effectiveness of our module design, we report the average forecasting results across seven datasets in Table~\ref{table:avg_module_ablation}. Our full model, \sysname, achieves the best average MSE on 5 out of 7 datasets and the best average MAE on 6 out of 7 datasets, consistently outperforming the two variants: \sysname\_{\text{advanced}}, which removes the basis evolution module, and \sysname\_{\text{base}}, which simplifies concept drift modeling. These quantitative improvements highlight the importance of jointly modeling both data drift and basis evolution to capture complex temporal dynamics for accurate forecasting.
We provide the full detailed forecasting results in Appendix (Section~\ref{appendix: full}), which further verify that \sysname attains superior performance in the majority of individual experiments, demonstrating its robustness and effectiveness.

\begin{table}[!t]
\centering
\caption{Average results of module ablation}
\label{table:avg_module_ablation}
\vspace{-9pt}
\small % 调整字体大小 small/scriptsize/footnotesize
\setlength{\tabcolsep}{4pt} % 调整列间距
\renewcommand{\arraystretch}{0.8} % 调整行间距
\begin{tabular}{c|cc|cc|cc}
\toprule
Model & \multicolumn{2}{c|}{\sysname} & \multicolumn{2}{c|}{\sysname\_adv.} & \multicolumn{2}{c}{\sysname\_base} \\ \midrule
Dataset & \multicolumn{1}{c|}{MSE} & MAE & \multicolumn{1}{c|}{MSE} & MAE & \multicolumn{1}{c|}{MSE} & MAE \\
ETTh1 & \multicolumn{1}{c|}{\textbf{0.425}} & \textbf{0.425} & \multicolumn{1}{c|}{0.431} & 0.430 & \multicolumn{1}{c|}{0.434} & 0.427 \\
ETTh2 & \multicolumn{1}{c|}{0.364} & \textbf{0.392} & \multicolumn{1}{c|}{\textbf{0.362}} & 0.392 & \multicolumn{1}{c|}{0.363} & 0.393 \\
ETTm1 & \multicolumn{1}{c|}{\textbf{0.375}} & \textbf{0.387} & \multicolumn{1}{c|}{0.375} & 0.391 & \multicolumn{1}{c|}{0.376} & 0.390 \\
ETTm2 & \multicolumn{1}{c|}{0.276} & 0.321 & \multicolumn{1}{c|}{0.277} & 0.322 & \multicolumn{1}{c|}{\textbf{0.275}} & \textbf{0.320} \\
Weather & \multicolumn{1}{c|}{\textbf{0.243}} & \textbf{0.269} & \multicolumn{1}{c|}{0.245} & 0.272 & \multicolumn{1}{c|}{0.246} & 0.272 \\
Traffic & \multicolumn{1}{c|}{\textbf{0.494}} & \textbf{0.278} & \multicolumn{1}{c|}{0.495} & 0.290 & \multicolumn{1}{c|}{0.506} & 0.308 \\
Elc & \multicolumn{1}{c|}{\textbf{0.175}} & \textbf{0.262} & \multicolumn{1}{c|}{0.178} & 0.264 & \multicolumn{1}{c|}{0.189} & 0.273 \\ \midrule
Best\_Count & \multicolumn{1}{c|}{\textbf{5/7}} & \textbf{6/7} & \multicolumn{1}{c|}{1/7} & 0/7 & \multicolumn{1}{c|}{1/7} & 1/7 \\ \bottomrule
\end{tabular}
\end{table}

%\para{Loss ablation.}  

We conduct an ablation study by progressively removing components of the loss function to evaluate their individual contributions. Specifically, \sysname\_{\text{enhanced}} removes the phase regulation term $\mathcal{R}_{\phi}$; \sysname\_{\text{advanced}} further removes the FFT loss $\mathcal{L}_{feq}$ based on \sysname\_{\text{base}}; and \sysname\_{\text{base}} discards all specialized loss designs, relying solely on the Huber loss.
Table~\ref{table:avg_loss_ablation} presents the average forecasting results. While the full model \sysname\ shows slightly better average MSE and MAE compared to \sysname\_{\text{enhanced}}, the full detailed results (see Appendix~\ref{appendix: full}) reveal that \sysname\ consistently outperforms all variants on a larger number of individual experiments. This indicates that although the average improvements appear modest, the full model demonstrates more substantial and consistent advantages in specific cases, highlighting the importance of each loss component for robust forecasting performance.

\begin{table}[!t]
\centering
\caption{Average results of loss ablation}
\label{table:avg_loss_ablation}
\vspace{-9pt}
\small % 调整字体大小 small/scriptsize/footnotesize
\setlength{\tabcolsep}{2.5pt} % 调整列间距
\renewcommand{\arraystretch}{0.8} % 调整行间距
\begin{tabular}{c|cc|cc|cc|cc}
\toprule
Model & \multicolumn{2}{c|}{\sysname} & \multicolumn{2}{c|}{\sysname\_enh.} & \multicolumn{2}{c|}{\sysname\_adv.} & \multicolumn{2}{c}{\sysname\_base} \\ \midrule
D
ataset & \multicolumn{1}{c|}{MSE} & MAE & \multicolumn{1}{c|}{MSE} & MAE & \multicolumn{1}{c|}{MSE} & MAE & \multicolumn{1}{c|}{MSE} & MAE \\
ETTh1 & \multicolumn{1}{c|}{\textbf{0.424}} & \textbf{0.424} & \multicolumn{1}{c|}{0.428} & 0.427 & \multicolumn{1}{c|}{0.439} & 0.437 & \multicolumn{1}{c|}{0.433} & 0.433 \\
ETTh2 & \multicolumn{1}{c|}{\textbf{0.363}} & 0.392 & \multicolumn{1}{c|}{0.363} & \textbf{0.391} & \multicolumn{1}{c|}{0.385} & 0.406 & \multicolumn{1}{c|}{0.367} & 0.394 \\
ETTm1 & \multicolumn{1}{c|}{0.374} & \textbf{0.386} & \multicolumn{1}{c|}{\textbf{0.374}} & 0.387 & \multicolumn{1}{c|}{0.384} & 0.401 & \multicolumn{1}{c|}{0.378} & 0.395 \\
ETTm2 & \multicolumn{1}{c|}{\textbf{0.276}} & 0.320 & \multicolumn{1}{c|}{0.277} & \textbf{0.319} & \multicolumn{1}{c|}{0.296} & 0.343 & \multicolumn{1}{c|}{0.282} & 0.327 \\
Weather & \multicolumn{1}{c|}{\textbf{0.243}} & 0.268 & \multicolumn{1}{c|}{0.243} & \textbf{0.267} & \multicolumn{1}{c|}{0.2448} & 0.2710 & \multicolumn{1}{c|}{0.2450} & 0.2705 \\
Traffic & \multicolumn{1}{c|}{0.494} & \textbf{0.277} & \multicolumn{1}{c|}{\textbf{0.487}} & 0.286 & \multicolumn{1}{c|}{0.509} & 0.287 & \multicolumn{1}{c|}{0.510} & 0.290 \\
Elc & \multicolumn{1}{c|}{0.175} & \textbf{0.262} & \multicolumn{1}{c|}{\textbf{0.174}} & 0.262 & \multicolumn{1}{c|}{0.180} & 0.270 & \multicolumn{1}{c|}{0.181} & 0.270 \\ \midrule
\multicolumn{1}{l|}{Best} & \multicolumn{1}{c|}{\textbf{4/7}} & \textbf{4/7} & \multicolumn{1}{c|}{3/7} & 3/7 & \multicolumn{1}{c|}{0} & 0 & \multicolumn{1}{c|}{0} & 0 \\ \bottomrule
\end{tabular}
\end{table}

\subsection{Concept Drift and Basis Evolution}

We quantify the degree of concept drift using ADWIN (Section~\ref{sec:data_drift}) and the degree of basis evolution (Section~\ref{sec:basis_evolution}). To evaluate the impact of these phenomena on model performance, we select two representative univariate time series: Weather\_d11 (dimension 11) and Traffic\_d738 (dimension 738). Weather\_d11 exhibits a concept drift degree of $3.07\%$ and a basis evolution degree of $8.39\%$, whereas Traffic\_d738 shows substantially lower degrees of $0.26\%$ and $1.19\%$, respectively.
We apply \sysname to these datasets and compare its forecasting accuracy against three SOTA frequency-domain models: FredFormer, WPMixer, and FITS. As shown in Table~\ref{table:drift_and_evol}, \sysname consistently outperforms these baselines, especially on Weather\_d11 where data drift and basis evolution are more pronounced. Specifically, on Weather\_d11, \sysname achieves an average MSE reduction of 7.0\% compared to FredFormer, 17.5\% compared to WPMixer, and 12.9\% compared to FITS. In terms of MAE, \sysname improves by about 3.5\%, 7.7\%, and 5.5\% over FredFormer, WPMixer, and FITS, respectively.
On the more stable Traffic\_d738 dataset, \sysname obtains the best average MSE (1.814), improving by 2.3\%, 1.2\%, and 2.1\% over FredFormer, WPMixer, and FITS, respectively. Regarding MAE, \sysname outperforms FredFormer and FITS by 7.1\% and 6.5\%, respectively, while WPMixer achieves a slightly better MAE (0.665) than \sysname (0.669) by about 0.6\%. These results demonstrate \sysname's superior adaptability and robustness in handling dynamic time series forecasting scenarios.

\begin{table}[!t]
\centering
\caption{Effectiveness of concept drift and basis evolution.}
\label{table:drift_and_evol}
\vspace{-9pt}
\small % 调整字体大小 small/scriptsize/footnotesize
\setlength{\tabcolsep}{2.5pt} % 调整列间距
\renewcommand{\arraystretch}{0.8} % 调整行间距
\begin{tabular}{c|c|cc|cc|cc|cc}
\toprule
Model &  & \multicolumn{2}{c|}{\sysname} & \multicolumn{2}{c|}{FredFormer} & \multicolumn{2}{c|}{Wpmixer} & \multicolumn{2}{c}{FITS} \\ \midrule
Dataset & T & \multicolumn{1}{c|}{MSE} & MAE & \multicolumn{1}{c|}{MSE} & MAE & \multicolumn{1}{c|}{MSE} & MAE & \multicolumn{1}{c|}{MSE} & MAE \\ \midrule
 & 96 & \multicolumn{1}{c|}{\textbf{0.110}} & \textbf{0.237} & \multicolumn{1}{c|}{0.131} & 0.260 & \multicolumn{1}{c|}{0.111} & 0.239 & \multicolumn{1}{c|}{0.127} & 0.257 \\
 & 192 & \multicolumn{1}{c|}{\textbf{0.185}} & \textbf{0.312} & \multicolumn{1}{c|}{0.203} & 0.326 & \multicolumn{1}{c|}{0.193} & 0.317 & \multicolumn{1}{c|}{0.200} & 0.322 \\
Weather-d11 & 336 & \multicolumn{1}{c|}{\textbf{0.302}} & 0.395 & \multicolumn{1}{c|}{0.321} & 0.405 & \multicolumn{1}{c|}{0.305} & \textbf{0.395} & \multicolumn{1}{c|}{0.317} & 0.401 \\
 & 720 & \multicolumn{1}{c|}{\textbf{0.462}} & 0.497 & \multicolumn{1}{c|}{0.481} & 0.503 & \multicolumn{1}{c|}{0.469} & \textbf{0.496} & \multicolumn{1}{c|}{0.478} & 0.501 \\
\multicolumn{1}{l|}{} & Avg. & \multicolumn{1}{c|}{\textbf{0.264}} & \textbf{0.360} & \multicolumn{1}{c|}{0.284} & 0.373 & \multicolumn{1}{c|}{0.269} & 0.362 & \multicolumn{1}{c|}{0.280} & 0.370 \\ \midrule
 & 96 & \multicolumn{1}{c|}{\textbf{1.854}} & \textbf{0.687} & \multicolumn{1}{c|}{1.882} & 0.742 & \multicolumn{1}{c|}{1.871} & 0.690 & \multicolumn{1}{c|}{1.921} & 0.741 \\
 & 192 & \multicolumn{1}{c|}{\textbf{1.898}} & 0.687 & \multicolumn{1}{c|}{1.969} & 0.746 & \multicolumn{1}{c|}{1.918} & \textbf{0.679} & \multicolumn{1}{c|}{1.951} & 0.729 \\
Traffic-d738 & 336 & \multicolumn{1}{c|}{\textbf{1.809}} & 0.665 & \multicolumn{1}{c|}{1.879} & 0.720 & \multicolumn{1}{c|}{1.815} & \textbf{0.645} & \multicolumn{1}{c|}{1.846} & 0.705 \\
 & 720 & \multicolumn{1}{c|}{\textbf{1.698}} & \textbf{0.639} & \multicolumn{1}{c|}{1.732} & 0.672 & \multicolumn{1}{c|}{1.742} & 0.646 & \multicolumn{1}{c|}{1.711} & 0.691 \\ 
\multicolumn{1}{l|}{} & Avg. & \multicolumn{1}{c|}{\textbf{1.814}} & 0.669 & \multicolumn{1}{c|}{1.865} & 0.720 & \multicolumn{1}{c|}{1.836} & \textbf{0.665} & \multicolumn{1}{c|}{1.857} & 0.716 \\ \bottomrule
\end{tabular}
\end{table}

\subsection{Scalability Analysis}

To investigate the scalability of \sysname, we train the model with increasing size from both the depth (number of layers) and width (embedding dimension) perspectives. Forecasting experiments are conducted on two datasets: Weather and Electricity. Figure~\ref{fig:scale_analysis} presents the average forecasting results, measured by Mean Squared Error (MSE), on both datasets for various forecasting horizons, including $T \in \{96, 192, 336, 720\}$ time steps, using different numbers of layers and embedding dimensions.

The results demonstrate that, unlike time series foundation models \cite{Das24, Goswami24, Woo24, Liu24, Zhang24}, time series forecasting models are typically trained on domain-specific datasets with limited data volume. As shown in Figure~\ref{fig:scale_analysis}, increasing model capacity---either by enlarging the hidden dimension or stacking more layers---yields diminishing returns after a certain point. Specifically, both the model dimension and layer analysis plots indicate that MSE saturates as the model size increases, and may even slightly worsen due to overfitting.
This phenomenon suggests that, for time series forecasting tasks with constrained data, the scalability of models is fundamentally limited. Once the model capacity matches the representational needs of the data, further scaling does not improve performance. This is in sharp contrast to Foundation Models, where scaling up with abundant data often leads to continuous performance gains.

\begin{figure}[!t]
    \centering
    \subfigure[Model dimension analysis]{
        \includegraphics[width=0.22\textwidth]{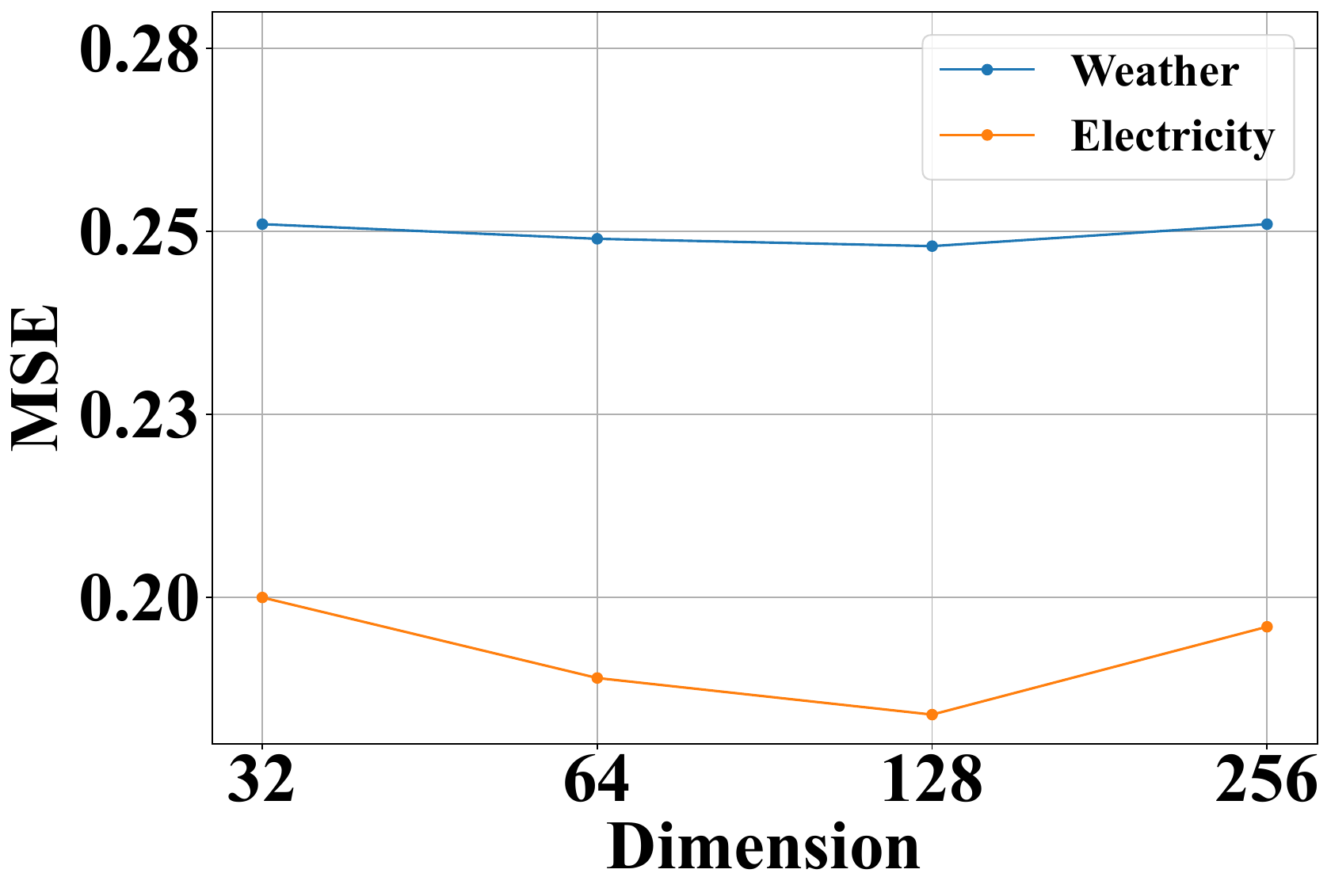}
    }
    \hfill
    \subfigure[Model layer analysis]{
        \includegraphics[width=0.22\textwidth]{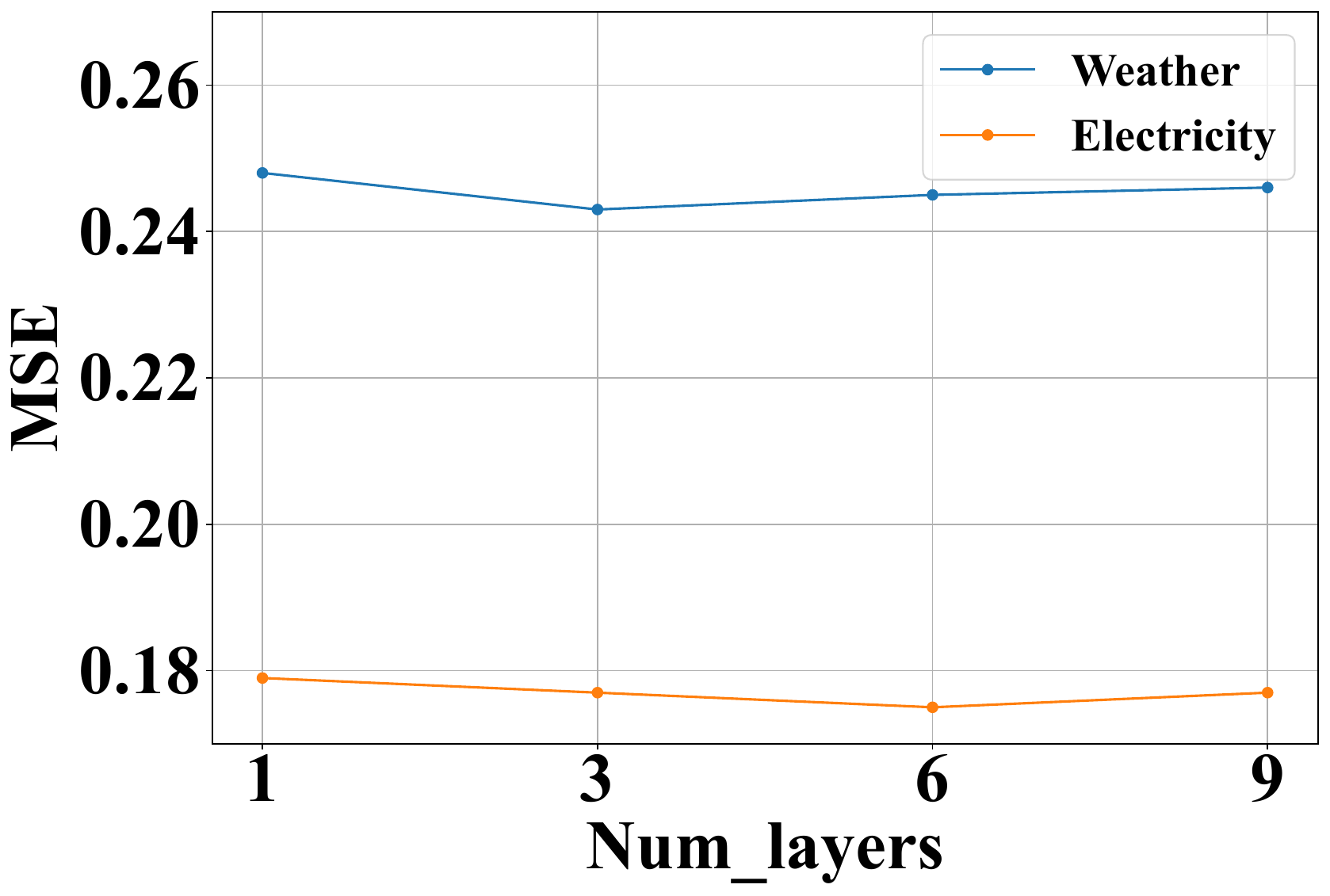}
    }
    \vspace{-12pt}
    \caption{Model scalability analysis}
    \label{fig:scale_analysis}
    \vspace{-9pt}
\end{figure}

\subsection{Hyper-parameter Analysis}

Patch length is a crucial hyper-parameter for \sysname. We evaluate the model's sensitivity to different patch lengths on the Weather and Electricity datasets, forecasting future time points \(T \in \{96, 192, 336, 720\}\).
Table~\ref{table:patch_len} reports the forecasting results measured by MSE and MAE. For the Weather dataset, the best overall performance is achieved with a patch length of 16, yielding an average MSE of 0.245 and MAE of 0.270. For the Electricity dataset, the optimal patch length is 32, with an average MSE of 0.175 and MAE of 0.262. 
Notably, the differences in performance across various patch lengths are marginal. For instance, on Weather, the worst average MSE (0.246 at patch length 4) is only 0.001 higher than the best (0.245 at patch length 16). Similarly, on Electricity, the average MSE varies within 0.006 across all tested patch lengths. This demonstrates that \sysname exhibits strong robustness and low sensitivity to patch length selection, consistent with the scalability analysis discussed earlier.

\begin{table}[!t]
\centering
\caption{Forecasting results of various patch lengths}
\label{table:patch_len}
\vspace{-9pt}
\scriptsize % 调整字体大小 small/scriptsize/footnotesize
\setlength{\tabcolsep}{3pt} % 调整列间距
\renewcommand{\arraystretch}{0.8} % 调整行间距
\begin{tabular}{c|c|cc|cc|cc|cc|cc}
\toprule
Patch Len &  & \multicolumn{2}{c|}{4} & \multicolumn{2}{c|}{8} & \multicolumn{2}{c|}{16} & \multicolumn{2}{c|}{32} & \multicolumn{2}{c}{48} \\ \midrule
Dataset & T & \multicolumn{1}{c|}{MSE} & MAE & \multicolumn{1}{c|}{MSE} & MAE & \multicolumn{1}{c|}{MSE} & MAE & \multicolumn{1}{c|}{MSE} & MAE & \multicolumn{1}{c|}{MSE} & MAE \\ \midrule
 & 96 & \multicolumn{1}{c|}{0.163} & 0.205 & \multicolumn{1}{c|}{0.163} & 0.204 & \multicolumn{1}{c|}{\textbf{0.162}} & 0.203 & \multicolumn{1}{c|}{0.163} & \textbf{0.202} & \multicolumn{1}{c|}{0.163} & 0.204 \\
 & 192 & \multicolumn{1}{c|}{0.210} & 0.248 & \multicolumn{1}{c|}{0.210} & 0.249 & \multicolumn{1}{c|}{\textbf{0.208}} & \textbf{0.246} & \multicolumn{1}{c|}{0.209} & 0.249 & \multicolumn{1}{c|}{0.209} & 0.246 \\
Weather & 336 & \multicolumn{1}{c|}{\textbf{0.266}} & \textbf{0.288} & \multicolumn{1}{c|}{0.267} & 0.290 & \multicolumn{1}{c|}{0.266} & 0.290 & \multicolumn{1}{c|}{0.268} & 0.289 & \multicolumn{1}{c|}{0.268} & 0.292 \\
 & 720 & \multicolumn{1}{c|}{0.343} & 0.342 & \multicolumn{1}{c|}{0.343} & 0.340 & \multicolumn{1}{c|}{\textbf{0.342}} & \textbf{0.339} & \multicolumn{1}{c|}{0.346} & 0.342 & \multicolumn{1}{c|}{0.344} & 0.341 \\
 & Avg. & \multicolumn{1}{c|}{0.246} & 0.271 & \multicolumn{1}{c|}{0.246} & 0.271 & \multicolumn{1}{c|}{\textbf{0.245}} & \textbf{0.270} & \multicolumn{1}{c|}{0.247} & 0.271 & \multicolumn{1}{c|}{0.246} & 0.271 \\ \midrule
 & 96 & \multicolumn{1}{c|}{0.154} & 0.243 & \multicolumn{1}{c|}{0.152} & 0.240 & \multicolumn{1}{c|}{0.149} & 0.237 & \multicolumn{1}{c|}{0.149} & 0.236 & \multicolumn{1}{c|}{\textbf{0.148}} & \textbf{0.235} \\
 & 192 & \multicolumn{1}{c|}{0.165} & 0.253 & \multicolumn{1}{c|}{0.163} & 0.250 & \multicolumn{1}{c|}{0.162} & 0.249 & \multicolumn{1}{c|}{\textbf{0.161}} & 0.249 & \multicolumn{1}{c|}{0.161} & \textbf{0.248} \\
Elc & 336 & \multicolumn{1}{c|}{0.180} & 0.270 & \multicolumn{1}{c|}{0.177} & 0.267 & \multicolumn{1}{c|}{0.177} & 0.267 & \multicolumn{1}{c|}{\textbf{0.176}} & \textbf{0.265} & \multicolumn{1}{c|}{0.178} & 0.268 \\
 & 720 & \multicolumn{1}{c|}{0.223} & 0.306 & \multicolumn{1}{c|}{0.217} & 0.301 & \multicolumn{1}{c|}{0.214} & 0.298 & \multicolumn{1}{c|}{\textbf{0.213}} & \textbf{0.298} & \multicolumn{1}{c|}{0.216} & 0.299 \\
 & Avg. & \multicolumn{1}{c|}{0.181} & 0.268 & \multicolumn{1}{c|}{0.177} & 0.265 & \multicolumn{1}{c|}{0.176} & 0.263 & \multicolumn{1}{c|}{\textbf{0.175}} & \textbf{0.262} & \multicolumn{1}{c|}{0.176} & 0.263 \\ \bottomrule
\end{tabular}
\end{table}
\section{Related Work}
\label{sec:related}

\noindent{\bf Time series forecasting and temporal models.} Time series forecasting presents unique challenges, especially in modeling long-term dependencies and complex temporal dynamics. Transformer-based architectures \citep{Vaswani17} have recently advanced the field by leveraging self-attention to capture global temporal relationships, outperforming traditional RNNs and CNNs \citep{Flunkert17, Lai18, Bai18, Cheng20} that often struggle with scalability and long-range modeling. Notable advancements include Informer \citep{Zhou21}, which introduces ProbSparse attention for efficient handling of long sequences, and Autoformer \citep{Wu21}, which decomposes time series into trend and seasonal components to improve interpretability and forecasting accuracy. PatchTST \citep{Nie23} restructures input sequences into patches for parallel processing in long-term prediction. Pyraformer \citep{Liu22} and TimesNet \citep{Wu23} explore hierarchical and multi-scale representations to further refine temporal modeling.

\para{Frequency-domain approaches.} Despite various advancements, time-domain models still fall short in capturing periodicity and spectral patterns inherent in many real-world time series. Frequency-domain approaches fill this void by leveraging Fourier and wavelet transforms to extract global and periodic features. Fredformer \citep{Piao24} employs frequency channel-wise attention to selectively focus on informative spectral components, while FreTS \citep{Yi23} models dependencies across frequency channels and temporal dimensions using MLPs. FITS \citep{Xu24} employs complex-valued layers for expressive frequency-domain transformations, and WPMixer \citep{Murad25} integrates wavelet decomposition with MLPs to capture both localized and long-term patterns. These models have demonstrated competitive or superior performance compared to purely temporal approaches.

\para{Hybrid temporal-frequency models.}  Recent studies have explored hybrid approaches that combine temporal and frequency-domain information. CDX-Net \citep{Li22b} integrates CNNs, RNNs, and attention mechanisms to extract and fuse multivariate features from both domains. FEDformer \citep{Zhou22} unifies trend-seasonal decomposition with Fourier analysis within a Transformer framework, enabling robust representation of multivariate time series. TimeMixer++ \citep{Wangsy24} generates multi-scale series via temporal downsampling, applies FFT-based periodic analysis, and employs attention mechanisms to learn robust representations of seasonal and trend components.

\para{Limitations of existing approaches.} Most frequency-domain and hybrid models, however, operate as black-box predictors, optimized primarily for accuracy with limited interpretability. Also, they rarely address practical challenges, such as concept drift and basis evolution, which undermine their robustness in dynamic environments where distributional shifts are common.

%To address these limitations, we propose \textbf{FreqFormer}, a principled framework that leverages FFT basis decomposition for unified and physically interpretable modeling. FreqFormer independently models amplitude and phase components, adaptively learns basis weights, and incorporates a novel loss function to explicitly address concept evolution and data drift. This design enhances both robustness and forecasting accuracy, while providing greater transparency into the model's decision process.
\section{Conclusion}
\label{sec:conclusion}

We use the discrete Fourier transform to unify the formulation of various types of time series. We propose \sysname, a new forecasting framework that works in the frequency domain through basis decomposition. This allows \sysname to capture richer, multi-dimensional features of temporal data.
A key strength of \sysname is its explicit and separate modeling of amplitude and phase for handling key challenges in time series forecasting, namely concept drift and basis evolution. \sysname combines rigorous mathematical ideas with practical components, including linear transformations, causal attention, and a composite loss function, so as to adapt dynamically and robustly to changing temporal patterns, even when data are noisy or sparse. Our experiments on diverse real-world datasets show that \sysname consistently achieves better accuracy, improved interpretability, and strong robustness. It performs especially well under severe concept drift and basis evolution, proving its effectiveness in dynamic scenarios.

For future work, we plan to strengthen the integration of mathematical theory with an interpretable model design for time series forecasting, move beyond trial-and-error methods, and develop more principled and transparent forecasting techniques, so as to tackle increasingly complex temporal data.

\end{sloppypar}

%\bibliographystyle{./ACM-Reference-Format}
%\bibliography{./references}

\begin{thebibliography}{45}

%%% ====================================================================
%%% NOTE TO THE USER: you can override these defaults by providing
%%% customized versions of any of these macros before the \bibliography
%%% command.  Each of them MUST provide its own final punctuation,
%%% except for \shownote{} and \showURL{}.  The latter two
%%% do not use final punctuation, in order to avoid confusing it with
%%% the Web address.
%%%
%%% To suppress output of a particular field, define its macro to expand
%%% to an empty string, or better, \unskip, like this:
%%%
%%% \newcommand{\showURL}[1]{\unskip}   % LaTeX syntax
%%%
%%% \def \showURL #1{\unskip}           % plain TeX syntax
%%%
%%% ====================================================================

\ifx \showCODEN    \undefined \def \showCODEN     #1{\unskip}     \fi
\ifx \showISBNx    \undefined \def \showISBNx     #1{\unskip}     \fi
\ifx \showISBNxiii \undefined \def \showISBNxiii  #1{\unskip}     \fi
\ifx \showISSN     \undefined \def \showISSN      #1{\unskip}     \fi
\ifx \showLCCN     \undefined \def \showLCCN      #1{\unskip}     \fi
\ifx \shownote     \undefined \def \shownote      #1{#1}          \fi
\ifx \showarticletitle \undefined \def \showarticletitle #1{#1}   \fi
\ifx \showURL      \undefined \def \showURL       {\relax}        \fi
% The following commands are used for tagged output and should be
% invisible to TeX
\providecommand\bibfield[2]{#2}
\providecommand\bibinfo[2]{#2}
\providecommand\natexlab[1]{#1}
\providecommand\showeprint[2][]{arXiv:#2}

\bibitem[Aiwansedo et~al\mbox{.}(2024)]%
        {Aiwansedo24}
\bibfield{author}{\bibinfo{person}{Konstandinos Aiwansedo}, \bibinfo{person}{J{\'{e}}r{\^{o}}me Bosche}, \bibinfo{person}{Wafa Badreddine}, \bibinfo{person}{Mohamed~Hamza Kermia}, {and} \bibinfo{person}{Oussama Djadane}.} \bibinfo{year}{2024}\natexlab{}.
\newblock \showarticletitle{{CNN-N-BEATS:} Novel Hybrid Model for Time-Series Forecasting}. In \bibinfo{booktitle}{\emph{Deep Learning Theory and Applications - 5th International Conference, DeLTA 2024, Dijon, France, July 10-11, 2024, Proceedings, Part {I}}} \emph{(\bibinfo{series}{Communications in Computer and Information Science}, Vol.~\bibinfo{volume}{2171})}, \bibfield{editor}{\bibinfo{person}{Ana Fred}, \bibinfo{person}{Allel Hadjali}, \bibinfo{person}{Oleg Gusikhin}, {and} \bibinfo{person}{Carlo Sansone}} (Eds.). \bibinfo{publisher}{Springer}, \bibinfo{pages}{38--57}.
\newblock
\href{https://doi.org/10.1007/978-3-031-66694-0\_3}{doi:\nolinkurl{10.1007/978-3-031-66694-0\_3}}


\bibitem[Bai et~al\mbox{.}(2018)]%
        {Bai18}
\bibfield{author}{\bibinfo{person}{Shaojie Bai}, \bibinfo{person}{J.~Zico Kolter}, {and} \bibinfo{person}{Vladlen Koltun}.} \bibinfo{year}{2018}\natexlab{}.
\newblock \showarticletitle{An Empirical Evaluation of Generic Convolutional and Recurrent Networks for Sequence Modeling}.
\newblock \bibinfo{journal}{\emph{CoRR}}  \bibinfo{volume}{abs/1803.01271} (\bibinfo{year}{2018}).
\newblock
\showeprint[arXiv]{1803.01271}


\bibitem[Bifet and Gavald{\`{a}}(2007)]%
        {Bifet07}
\bibfield{author}{\bibinfo{person}{Albert Bifet} {and} \bibinfo{person}{Ricard Gavald{\`{a}}}.} \bibinfo{year}{2007}\natexlab{}.
\newblock \showarticletitle{Learning from Time-Changing Data with Adaptive Windowing}. In \bibinfo{booktitle}{\emph{Proceedings of the Seventh {SIAM} International Conference on Data Mining, April 26-28, 2007, Minneapolis, Minnesota, {USA}}}. \bibinfo{publisher}{{SIAM}}, \bibinfo{pages}{443--448}.
\newblock
\href{https://doi.org/10.1137/1.9781611972771.42}{doi:\nolinkurl{10.1137/1.9781611972771.42}}


\bibitem[Boulegane et~al\mbox{.}(2022)]%
        {Boulegane22}
\bibfield{author}{\bibinfo{person}{Dihia Boulegane}, \bibinfo{person}{Vitor Cerquiera}, {and} \bibinfo{person}{Albert Bifet}.} \bibinfo{year}{2022}\natexlab{}.
\newblock \showarticletitle{Adaptive Model Compression of Ensembles for Evolving Data Streams Forecasting}. In \bibinfo{booktitle}{\emph{2022 International Joint Conference on Neural Networks (IJCNN)}}. \bibinfo{pages}{1--8}.
\newblock
\href{https://doi.org/10.1109/IJCNN55064.2022.9892811}{doi:\nolinkurl{10.1109/IJCNN55064.2022.9892811}}


\bibitem[Cheng et~al\mbox{.}(2020)]%
        {Cheng20}
\bibfield{author}{\bibinfo{person}{Jiezhu Cheng}, \bibinfo{person}{Kaizhu Huang}, {and} \bibinfo{person}{Zibin Zheng}.} \bibinfo{year}{2020}\natexlab{}.
\newblock \showarticletitle{Towards Better Forecasting by Fusing Near and Distant Future Visions}. In \bibinfo{booktitle}{\emph{The Thirty-Fourth {AAAI} Conference on Artificial Intelligence, {AAAI} 2020, The Thirty-Second Innovative Applications of Artificial Intelligence Conference, {IAAI} 2020, The Tenth {AAAI} Symposium on Educational Advances in Artificial Intelligence, {EAAI} 2020, New York, NY, USA, February 7-12, 2020}}. \bibinfo{publisher}{{AAAI} Press}, \bibinfo{pages}{3593--3600}.
\newblock


\bibitem[Das et~al\mbox{.}(2024)]%
        {Das24}
\bibfield{author}{\bibinfo{person}{Abhimanyu Das}, \bibinfo{person}{Weihao Kong}, \bibinfo{person}{Rajat Sen}, {and} \bibinfo{person}{Yichen Zhou}.} \bibinfo{year}{2024}\natexlab{}.
\newblock \showarticletitle{A decoder-only foundation model for time-series forecasting}. In \bibinfo{booktitle}{\emph{Forty-first International Conference on Machine Learning, {ICML} 2024, Vienna, Austria, July 21-27, 2024}}. \bibinfo{publisher}{OpenReview.net}.
\newblock
\urldef\tempurl%
\url{https://openreview.net/forum?id=jn2iTJas6h}
\showURL{%
\tempurl}


\bibitem[Du et~al\mbox{.}(2023)]%
        {Du23}
\bibfield{author}{\bibinfo{person}{Wenjie Du}, \bibinfo{person}{David C{\^{o}}t{\'{e}}}, {and} \bibinfo{person}{Yan Liu}.} \bibinfo{year}{2023}\natexlab{}.
\newblock \showarticletitle{{SAITS:} Self-attention-based imputation for time series}.
\newblock \bibinfo{journal}{\emph{Expert Syst. Appl.}}  \bibinfo{volume}{219} (\bibinfo{year}{2023}), \bibinfo{pages}{119619}.
\newblock
\href{https://doi.org/10.1016/J.ESWA.2023.119619}{doi:\nolinkurl{10.1016/J.ESWA.2023.119619}}


\bibitem[Durairaj and Mohan(2022)]%
        {Durairaj22}
\bibfield{author}{\bibinfo{person}{M. Durairaj} {and} \bibinfo{person}{B.~H.~Krishna Mohan}.} \bibinfo{year}{2022}\natexlab{}.
\newblock \showarticletitle{A convolutional neural network based approach to financial time series prediction}.
\newblock \bibinfo{journal}{\emph{Neural Comput. Appl.}} \bibinfo{volume}{34}, \bibinfo{number}{16} (\bibinfo{year}{2022}), \bibinfo{pages}{13319--13337}.
\newblock
\href{https://doi.org/10.1007/S00521-022-07143-2}{doi:\nolinkurl{10.1007/S00521-022-07143-2}}


\bibitem[Flunkert et~al\mbox{.}(2017)]%
        {Flunkert17}
\bibfield{author}{\bibinfo{person}{Valentin Flunkert}, \bibinfo{person}{David Salinas}, {and} \bibinfo{person}{Jan Gasthaus}.} \bibinfo{year}{2017}\natexlab{}.
\newblock \showarticletitle{DeepAR: Probabilistic Forecasting with Autoregressive Recurrent Networks}.
\newblock \bibinfo{journal}{\emph{CoRR}}  \bibinfo{volume}{abs/1704.04110} (\bibinfo{year}{2017}).
\newblock
\showeprint[arXiv]{1704.04110}


\bibitem[Gama et~al\mbox{.}(2014)]%
        {Gama14}
\bibfield{author}{\bibinfo{person}{Jo{\~a}o Gama}, \bibinfo{person}{Indr{\.e} {\v{Z}}liobait{\.e}}, \bibinfo{person}{Albert Bifet}, \bibinfo{person}{Mykola Pechenizkiy}, {and} \bibinfo{person}{Abdelhamid Bouchachia}.} \bibinfo{year}{2014}\natexlab{}.
\newblock \showarticletitle{A survey on concept drift adaptation}.
\newblock \bibinfo{journal}{\emph{ACM computing surveys (CSUR)}} \bibinfo{volume}{46}, \bibinfo{number}{4} (\bibinfo{year}{2014}), \bibinfo{pages}{1--37}.
\newblock


\bibitem[Goswami et~al\mbox{.}(2024)]%
        {Goswami24}
\bibfield{author}{\bibinfo{person}{Mononito Goswami}, \bibinfo{person}{Konrad Szafer}, \bibinfo{person}{Arjun Choudhry}, \bibinfo{person}{Yifu Cai}, \bibinfo{person}{Shuo Li}, {and} \bibinfo{person}{Artur Dubrawski}.} \bibinfo{year}{2024}\natexlab{}.
\newblock \showarticletitle{{MOMENT:} {A} Family of Open Time-series Foundation Models}. In \bibinfo{booktitle}{\emph{Forty-first International Conference on Machine Learning, {ICML} 2024, Vienna, Austria, July 21-27, 2024}}. \bibinfo{publisher}{OpenReview.net}.
\newblock
\urldef\tempurl%
\url{https://openreview.net/forum?id=FVvf69a5rx}
\showURL{%
\tempurl}


\bibitem[Gunasekara et~al\mbox{.}(2024)]%
        {Gunasekara24}
\bibfield{author}{\bibinfo{person}{Nuwan Gunasekara}, \bibinfo{person}{Bernhard Pfahringer}, \bibinfo{person}{Heitor~Murilo Gomes}, \bibinfo{person}{Albert Bifet}, {and} \bibinfo{person}{Yun~Sing Koh}.} \bibinfo{year}{2024}\natexlab{}.
\newblock \showarticletitle{Recurrent concept drifts on data streams}. International Joint Conferences on Artificial Intelligence Organization.
\newblock


\bibitem[Gurjar and Chhabria(2015)]%
        {Gurjar15}
\bibfield{author}{\bibinfo{person}{Gajendra~Singh Gurjar} {and} \bibinfo{person}{Sharda Chhabria}.} \bibinfo{year}{2015}\natexlab{}.
\newblock \showarticletitle{A review on concept evolution technique on data stream}. In \bibinfo{booktitle}{\emph{2015 International Conference on Pervasive Computing (ICPC)}}. \bibinfo{pages}{1--3}.
\newblock
\href{https://doi.org/10.1109/PERVASIVE.2015.7087172}{doi:\nolinkurl{10.1109/PERVASIVE.2015.7087172}}


\bibitem[Haque et~al\mbox{.}(2016)]%
        {Haque16}
\bibfield{author}{\bibinfo{person}{Ahsanul Haque}, \bibinfo{person}{Latifur Khan}, \bibinfo{person}{Michael Baron}, \bibinfo{person}{Bhavani Thuraisingham}, {and} \bibinfo{person}{Charu Aggarwal}.} \bibinfo{year}{2016}\natexlab{}.
\newblock \showarticletitle{Efficient handling of concept drift and concept evolution over Stream Data}. In \bibinfo{booktitle}{\emph{2016 IEEE 32nd International Conference on Data Engineering (ICDE)}}. \bibinfo{pages}{481--492}.
\newblock
\href{https://doi.org/10.1109/ICDE.2016.7498264}{doi:\nolinkurl{10.1109/ICDE.2016.7498264}}


\bibitem[Huber(1981)]%
        {Huber81}
\bibfield{author}{\bibinfo{person}{Peter~J. Huber}.} \bibinfo{year}{1981}\natexlab{}.
\newblock \bibinfo{booktitle}{\emph{Robust Statistics}}.
\newblock \bibinfo{publisher}{Wiley}.
\newblock
\showISBNx{978-0-47141805-4}
\href{https://doi.org/10.1002/0471725250}{doi:\nolinkurl{10.1002/0471725250}}


\bibitem[Lai et~al\mbox{.}(2018)]%
        {Lai18}
\bibfield{author}{\bibinfo{person}{Guokun Lai}, \bibinfo{person}{Wei{-}Cheng Chang}, \bibinfo{person}{Yiming Yang}, {and} \bibinfo{person}{Hanxiao Liu}.} \bibinfo{year}{2018}\natexlab{}.
\newblock \showarticletitle{Modeling Long- and Short-Term Temporal Patterns with Deep Neural Networks}. In \bibinfo{booktitle}{\emph{The 41st International {ACM} {SIGIR} Conference on Research {\&} Development in Information Retrieval, {SIGIR} 2018, Ann Arbor, MI, USA, July 08-12, 2018}}, \bibfield{editor}{\bibinfo{person}{Kevyn Collins{-}Thompson}, \bibinfo{person}{Qiaozhu Mei}, \bibinfo{person}{Brian~D. Davison}, \bibinfo{person}{Yiqun Liu}, {and} \bibinfo{person}{Emine Yilmaz}} (Eds.). \bibinfo{publisher}{{ACM}}, \bibinfo{pages}{95--104}.
\newblock


\bibitem[Li et~al\mbox{.}(2025)]%
        {Li25}
\bibfield{author}{\bibinfo{person}{Chun{-}Na Li}, \bibinfo{person}{Yiwei Song}, {and} \bibinfo{person}{Yuan{-}Hai Shao}.} \bibinfo{year}{2025}\natexlab{}.
\newblock \showarticletitle{Domain Adaptation via Learning Using Statistical Invariant}.
\newblock \bibinfo{journal}{\emph{{IEEE} Trans. Knowl. Data Eng.}} \bibinfo{volume}{37}, \bibinfo{number}{7} (\bibinfo{year}{2025}), \bibinfo{pages}{4023--4034}.
\newblock
\href{https://doi.org/10.1109/TKDE.2025.3565780}{doi:\nolinkurl{10.1109/TKDE.2025.3565780}}


\bibitem[Li et~al\mbox{.}(2022a)]%
        {Li22b}
\bibfield{author}{\bibinfo{person}{Jiajia Li}, \bibinfo{person}{Ling Dai}, \bibinfo{person}{Feng Tan}, \bibinfo{person}{Hui Shen}, \bibinfo{person}{Zikai Wang}, \bibinfo{person}{Bin Sheng}, {and} \bibinfo{person}{Pengwei Hu}.} \bibinfo{year}{2022}\natexlab{a}.
\newblock \showarticletitle{{CDX-NET:} Cross-Domain Multi-Feature Fusion Modeling Via Deep Neural Networks for Multivariate Time Series Forecasting in AIOps}. In \bibinfo{booktitle}{\emph{{IEEE} International Conference on Acoustics, Speech and Signal Processing, {ICASSP} 2022, Virtual and Singapore, 23-27 May 2022}}. \bibinfo{publisher}{{IEEE}}, \bibinfo{pages}{4073--4077}.
\newblock
\href{https://doi.org/10.1109/ICASSP43922.2022.9746242}{doi:\nolinkurl{10.1109/ICASSP43922.2022.9746242}}


\bibitem[Li et~al\mbox{.}(2022b)]%
        {Li22a}
\bibfield{author}{\bibinfo{person}{Jiajia Li}, \bibinfo{person}{Feng Tan}, \bibinfo{person}{Cheng He}, \bibinfo{person}{Zikai Wang}, \bibinfo{person}{Haitao Song}, \bibinfo{person}{Lingfei Wu}, {and} \bibinfo{person}{Pengwei Hu}.} \bibinfo{year}{2022}\natexlab{b}.
\newblock \showarticletitle{HigeNet: {A} Highly Efficient Modeling for Long Sequence Time Series Prediction in AIOps}.
\newblock \bibinfo{journal}{\emph{CoRR}}  \bibinfo{volume}{abs/2211.07642} (\bibinfo{year}{2022}).
\newblock
\showeprint[arXiv]{2211.07642}
\href{https://doi.org/10.48550/ARXIV.2211.07642}{doi:\nolinkurl{10.48550/ARXIV.2211.07642}}


\bibitem[Liu et~al\mbox{.}(2022)]%
        {Liu22}
\bibfield{author}{\bibinfo{person}{Shizhan Liu}, \bibinfo{person}{Hang Yu}, \bibinfo{person}{Cong Liao}, \bibinfo{person}{Jianguo Li}, \bibinfo{person}{Weiyao Lin}, \bibinfo{person}{Alex~X. Liu}, {and} \bibinfo{person}{Schahram Dustdar}.} \bibinfo{year}{2022}\natexlab{}.
\newblock \showarticletitle{Pyraformer: Low-Complexity Pyramidal Attention for Long-Range Time Series Modeling and Forecasting}. In \bibinfo{booktitle}{\emph{The Tenth International Conference on Learning Representations, {ICLR} 2022, Virtual Event, April 25-29, 2022}}. \bibinfo{publisher}{OpenReview.net}.
\newblock
\urldef\tempurl%
\url{https://openreview.net/forum?id=0EXmFzUn5I}
\showURL{%
\tempurl}


\bibitem[Liu et~al\mbox{.}(2024a)]%
        {LiuH24}
\bibfield{author}{\bibinfo{person}{Yong Liu}, \bibinfo{person}{Tengge Hu}, \bibinfo{person}{Haoran Zhang}, \bibinfo{person}{Haixu Wu}, \bibinfo{person}{Shiyu Wang}, \bibinfo{person}{Lintao Ma}, {and} \bibinfo{person}{Mingsheng Long}.} \bibinfo{year}{2024}\natexlab{a}.
\newblock \showarticletitle{iTransformer: Inverted Transformers Are Effective for Time Series Forecasting}. In \bibinfo{booktitle}{\emph{The Twelfth International Conference on Learning Representations, {ICLR} 2024, Vienna, Austria, May 7-11, 2024}}. \bibinfo{publisher}{OpenReview.net}.
\newblock
\urldef\tempurl%
\url{https://openreview.net/forum?id=JePfAI8fah}
\showURL{%
\tempurl}


\bibitem[Liu et~al\mbox{.}(2024b)]%
        {Liu24}
\bibfield{author}{\bibinfo{person}{Yong Liu}, \bibinfo{person}{Haoran Zhang}, \bibinfo{person}{Chenyu Li}, \bibinfo{person}{Xiangdong Huang}, \bibinfo{person}{Jianmin Wang}, {and} \bibinfo{person}{Mingsheng Long}.} \bibinfo{year}{2024}\natexlab{b}.
\newblock \showarticletitle{Timer: Generative Pre-trained Transformers Are Large Time Series Models}. In \bibinfo{booktitle}{\emph{Forty-first International Conference on Machine Learning, {ICML} 2024, Vienna, Austria, July 21-27, 2024}}. \bibinfo{publisher}{OpenReview.net}.
\newblock
\urldef\tempurl%
\url{https://openreview.net/forum?id=bYRYb7DMNo}
\showURL{%
\tempurl}


\bibitem[Masud et~al\mbox{.}(2010)]%
        {Masud10}
\bibfield{author}{\bibinfo{person}{Mohammad~M. Masud}, \bibinfo{person}{Qing Chen}, \bibinfo{person}{Latifur Khan}, \bibinfo{person}{Charu~C. Aggarwal}, \bibinfo{person}{Jing Gao}, \bibinfo{person}{Jiawei Han}, {and} \bibinfo{person}{Bhavani Thuraisingham}.} \bibinfo{year}{2010}\natexlab{}.
\newblock \showarticletitle{Addressing Concept-Evolution in Concept-Drifting Data Streams}. In \bibinfo{booktitle}{\emph{{ICDM} 2010, The 10th {IEEE} International Conference on Data Mining, Sydney, Australia, 14-17 December 2010}}, \bibfield{editor}{\bibinfo{person}{Geoffrey~I. Webb}, \bibinfo{person}{Bing Liu}, \bibinfo{person}{Chengqi Zhang}, \bibinfo{person}{Dimitrios Gunopulos}, {and} \bibinfo{person}{Xindong Wu}} (Eds.). \bibinfo{publisher}{{IEEE} Computer Society}, \bibinfo{pages}{929--934}.
\newblock
\href{https://doi.org/10.1109/ICDM.2010.160}{doi:\nolinkurl{10.1109/ICDM.2010.160}}


\bibitem[Murad et~al\mbox{.}(2025)]%
        {Murad25}
\bibfield{author}{\bibinfo{person}{Md~Mahmuddun~Nabi Murad}, \bibinfo{person}{Mehmet Aktukmak}, {and} \bibinfo{person}{Yasin Yilmaz}.} \bibinfo{year}{2025}\natexlab{}.
\newblock \showarticletitle{WPMixer: Efficient Multi-Resolution Mixing for Long-Term Time Series Forecasting}. In \bibinfo{booktitle}{\emph{AAAI-25, Sponsored by the Association for the Advancement of Artificial Intelligence, February 25 - March 4, 2025, Philadelphia, PA, {USA}}}, \bibfield{editor}{\bibinfo{person}{Toby Walsh}, \bibinfo{person}{Julie Shah}, {and} \bibinfo{person}{Zico Kolter}} (Eds.). \bibinfo{publisher}{{AAAI} Press}, \bibinfo{pages}{19581--19588}.
\newblock
\href{https://doi.org/10.1609/AAAI.V39I18.34156}{doi:\nolinkurl{10.1609/AAAI.V39I18.34156}}


\bibitem[Nie et~al\mbox{.}(2023)]%
        {Nie23}
\bibfield{author}{\bibinfo{person}{Yuqi Nie}, \bibinfo{person}{Nam~H. Nguyen}, \bibinfo{person}{Phanwadee Sinthong}, {and} \bibinfo{person}{Jayant Kalagnanam}.} \bibinfo{year}{2023}\natexlab{}.
\newblock \showarticletitle{A Time Series is Worth 64 Words: Long-term Forecasting with Transformers}. In \bibinfo{booktitle}{\emph{The Eleventh International Conference on Learning Representations, {ICLR} 2023, Kigali, Rwanda, May 1-5, 2023}}. \bibinfo{publisher}{OpenReview.net}.
\newblock
\urldef\tempurl%
\url{https://openreview.net/forum?id=Jbdc0vTOcol}
\showURL{%
\tempurl}


\bibitem[Paszke et~al\mbox{.}(2019)]%
        {Paszke19}
\bibfield{author}{\bibinfo{person}{Adam Paszke}, \bibinfo{person}{Sam Gross}, \bibinfo{person}{Francisco Massa}, \bibinfo{person}{Adam Lerer}, \bibinfo{person}{James Bradbury}, \bibinfo{person}{Gregory Chanan}, \bibinfo{person}{Trevor Killeen}, \bibinfo{person}{Zeming Lin}, \bibinfo{person}{Natalia Gimelshein}, \bibinfo{person}{Luca Antiga}, \bibinfo{person}{Alban Desmaison}, \bibinfo{person}{Andreas K{\"{o}}pf}, \bibinfo{person}{Edward~Z. Yang}, \bibinfo{person}{Zachary DeVito}, \bibinfo{person}{Martin Raison}, \bibinfo{person}{Alykhan Tejani}, \bibinfo{person}{Sasank Chilamkurthy}, \bibinfo{person}{Benoit Steiner}, \bibinfo{person}{Lu Fang}, \bibinfo{person}{Junjie Bai}, {and} \bibinfo{person}{Soumith Chintala}.} \bibinfo{year}{2019}\natexlab{}.
\newblock \showarticletitle{PyTorch: An Imperative Style, High-Performance Deep Learning Library}. In \bibinfo{booktitle}{\emph{Advances in Neural Information Processing Systems 32: Annual Conference on Neural Information Processing Systems 2019, NeurIPS 2019, December 8-14, 2019, Vancouver, BC, Canada}}, \bibfield{editor}{\bibinfo{person}{Hanna~M. Wallach}, \bibinfo{person}{Hugo Larochelle}, \bibinfo{person}{Alina Beygelzimer}, \bibinfo{person}{Florence d'Alch{\'{e}}{-}Buc}, \bibinfo{person}{Emily~B. Fox}, {and} \bibinfo{person}{Roman Garnett}} (Eds.). \bibinfo{pages}{8024--8035}.
\newblock
\urldef\tempurl%
\url{https://proceedings.neurips.cc/paper/2019/hash/bdbca288fee7f92f2bfa9f7012727740-Abstract.html}
\showURL{%
\tempurl}


\bibitem[Piao et~al\mbox{.}(2024)]%
        {Piao24}
\bibfield{author}{\bibinfo{person}{Xihao Piao}, \bibinfo{person}{Zheng Chen}, \bibinfo{person}{Taichi Murayama}, \bibinfo{person}{Yasuko Matsubara}, {and} \bibinfo{person}{Yasushi Sakurai}.} \bibinfo{year}{2024}\natexlab{}.
\newblock \showarticletitle{Fredformer: Frequency Debiased Transformer for Time Series Forecasting}. In \bibinfo{booktitle}{\emph{Proceedings of the 30th {ACM} {SIGKDD} Conference on Knowledge Discovery and Data Mining, {KDD} 2024, Barcelona, Spain, August 25-29, 2024}}, \bibfield{editor}{\bibinfo{person}{Ricardo Baeza{-}Yates} {and} \bibinfo{person}{Francesco Bonchi}} (Eds.). \bibinfo{publisher}{{ACM}}, \bibinfo{pages}{2400--2410}.
\newblock
\href{https://doi.org/10.1145/3637528.3671928}{doi:\nolinkurl{10.1145/3637528.3671928}}


\bibitem[Song et~al\mbox{.}(2023)]%
        {Song23}
\bibfield{author}{\bibinfo{person}{Junho Song}, \bibinfo{person}{Keonwoo Kim}, \bibinfo{person}{Jeonglyul Oh}, {and} \bibinfo{person}{Sungzoon Cho}.} \bibinfo{year}{2023}\natexlab{}.
\newblock \showarticletitle{{MEMTO:} Memory-guided Transformer for Multivariate Time Series Anomaly Detection}. In \bibinfo{booktitle}{\emph{Advances in Neural Information Processing Systems 36: Annual Conference on Neural Information Processing Systems 2023, NeurIPS 2023, New Orleans, LA, USA, December 10 - 16, 2023}}, \bibfield{editor}{\bibinfo{person}{Alice Oh}, \bibinfo{person}{Tristan Naumann}, \bibinfo{person}{Amir Globerson}, \bibinfo{person}{Kate Saenko}, \bibinfo{person}{Moritz Hardt}, {and} \bibinfo{person}{Sergey Levine}} (Eds.).
\newblock
\urldef\tempurl%
\url{http://papers.nips.cc/paper\_files/paper/2023/hash/b4c898eb1fb556b8d871fbe9ead92256-Abstract-Conference.html}
\showURL{%
\tempurl}


\bibitem[Tsymbal(2004)]%
        {Tsymbal04}
\bibfield{author}{\bibinfo{person}{Alexey Tsymbal}.} \bibinfo{year}{2004}\natexlab{}.
\newblock \showarticletitle{The problem of concept drift: definitions and related work}.
\newblock \bibinfo{journal}{\emph{Computer Science Department, Trinity College Dublin}} \bibinfo{volume}{106}, \bibinfo{number}{2} (\bibinfo{year}{2004}), \bibinfo{pages}{58}.
\newblock


\bibitem[Vapnik and Izmailov(2020)]%
        {Vapnik20}
\bibfield{author}{\bibinfo{person}{Vladimir Vapnik} {and} \bibinfo{person}{Rauf Izmailov}.} \bibinfo{year}{2020}\natexlab{}.
\newblock \showarticletitle{Complete statistical theory of learning: learning using statistical invariants}. In \bibinfo{booktitle}{\emph{Conformal and Probabilistic Prediction and Applications, {COPA} 2020, 9-11 September 2020, Virtual Event}} \emph{(\bibinfo{series}{Proceedings of Machine Learning Research}, Vol.~\bibinfo{volume}{128})}, \bibfield{editor}{\bibinfo{person}{Alexander Gammerman}, \bibinfo{person}{Vladimir Vovk}, \bibinfo{person}{Zhiyuan Luo}, \bibinfo{person}{Evgueni~N. Smirnov}, \bibinfo{person}{Giovanni Cherubin}, {and} \bibinfo{person}{Marco Christini}} (Eds.). \bibinfo{publisher}{{PMLR}}, \bibinfo{pages}{4--40}.
\newblock
\urldef\tempurl%
\url{http://proceedings.mlr.press/v128/vapnik20a.html}
\showURL{%
\tempurl}


\bibitem[Vaswani et~al\mbox{.}(2017)]%
        {Vaswani17}
\bibfield{author}{\bibinfo{person}{Ashish Vaswani}, \bibinfo{person}{Noam Shazeer}, \bibinfo{person}{Niki Parmar}, \bibinfo{person}{Jakob Uszkoreit}, \bibinfo{person}{Llion Jones}, \bibinfo{person}{Aidan~N. Gomez}, \bibinfo{person}{Lukasz Kaiser}, {and} \bibinfo{person}{Illia Polosukhin}.} \bibinfo{year}{2017}\natexlab{}.
\newblock \showarticletitle{Attention is All you Need}. In \bibinfo{booktitle}{\emph{Advances in Neural Information Processing Systems 30: Annual Conference on Neural Information Processing Systems 2017, December 4-9, 2017, Long Beach, CA, {USA}}}, \bibfield{editor}{\bibinfo{person}{Isabelle Guyon}, \bibinfo{person}{Ulrike von Luxburg}, \bibinfo{person}{Samy Bengio}, \bibinfo{person}{Hanna~M. Wallach}, \bibinfo{person}{Rob Fergus}, \bibinfo{person}{S.~V.~N. Vishwanathan}, {and} \bibinfo{person}{Roman Garnett}} (Eds.). \bibinfo{pages}{5998--6008}.
\newblock


\bibitem[Wang et~al\mbox{.}(2024a)]%
        {Wangsy24}
\bibfield{author}{\bibinfo{person}{Shiyu Wang}, \bibinfo{person}{Jiawei Li}, \bibinfo{person}{Xiaoming Shi}, \bibinfo{person}{Zhou Ye}, \bibinfo{person}{Baichuan Mo}, \bibinfo{person}{Wenze Lin}, \bibinfo{person}{Shengtong Ju}, \bibinfo{person}{Zhixuan Chu}, {and} \bibinfo{person}{Ming Jin}.} \bibinfo{year}{2024}\natexlab{a}.
\newblock \showarticletitle{TimeMixer++: {A} General Time Series Pattern Machine for Universal Predictive Analysis}.
\newblock \bibinfo{journal}{\emph{CoRR}}  \bibinfo{volume}{abs/2410.16032} (\bibinfo{year}{2024}).
\newblock
\showeprint[arXiv]{2410.16032}
\href{https://doi.org/10.48550/ARXIV.2410.16032}{doi:\nolinkurl{10.48550/ARXIV.2410.16032}}


\bibitem[Wang et~al\mbox{.}(2024b)]%
        {Wang24}
\bibfield{author}{\bibinfo{person}{Shiyu Wang}, \bibinfo{person}{Haixu Wu}, \bibinfo{person}{Xiaoming Shi}, \bibinfo{person}{Tengge Hu}, \bibinfo{person}{Huakun Luo}, \bibinfo{person}{Lintao Ma}, \bibinfo{person}{James~Y. Zhang}, {and} \bibinfo{person}{Jun Zhou}.} \bibinfo{year}{2024}\natexlab{b}.
\newblock \showarticletitle{TimeMixer: Decomposable Multiscale Mixing for Time Series Forecasting}. In \bibinfo{booktitle}{\emph{The Twelfth International Conference on Learning Representations, {ICLR} 2024, Vienna, Austria, May 7-11, 2024}}. \bibinfo{publisher}{OpenReview.net}.
\newblock
\urldef\tempurl%
\url{https://openreview.net/forum?id=7oLshfEIC2}
\showURL{%
\tempurl}


\bibitem[Woo et~al\mbox{.}(2024)]%
        {Woo24}
\bibfield{author}{\bibinfo{person}{Gerald Woo}, \bibinfo{person}{Chenghao Liu}, \bibinfo{person}{Akshat Kumar}, \bibinfo{person}{Caiming Xiong}, \bibinfo{person}{Silvio Savarese}, {and} \bibinfo{person}{Doyen Sahoo}.} \bibinfo{year}{2024}\natexlab{}.
\newblock \showarticletitle{Unified Training of Universal Time Series Forecasting Transformers}. In \bibinfo{booktitle}{\emph{Forty-first International Conference on Machine Learning, {ICML} 2024, Vienna, Austria, July 21-27, 2024}}. \bibinfo{publisher}{OpenReview.net}.
\newblock
\urldef\tempurl%
\url{https://openreview.net/forum?id=Yd8eHMY1wz}
\showURL{%
\tempurl}


\bibitem[Wu et~al\mbox{.}(2023)]%
        {Wu23}
\bibfield{author}{\bibinfo{person}{Haixu Wu}, \bibinfo{person}{Tengge Hu}, \bibinfo{person}{Yong Liu}, \bibinfo{person}{Hang Zhou}, \bibinfo{person}{Jianmin Wang}, {and} \bibinfo{person}{Mingsheng Long}.} \bibinfo{year}{2023}\natexlab{}.
\newblock \showarticletitle{TimesNet: Temporal 2D-Variation Modeling for General Time Series Analysis}. In \bibinfo{booktitle}{\emph{The Eleventh International Conference on Learning Representations, {ICLR} 2023, Kigali, Rwanda, May 1-5, 2023}}. \bibinfo{publisher}{OpenReview.net}.
\newblock
\urldef\tempurl%
\url{https://openreview.net/forum?id=ju\_Uqw384Oq}
\showURL{%
\tempurl}


\bibitem[Wu et~al\mbox{.}(2021)]%
        {Wu21}
\bibfield{author}{\bibinfo{person}{Haixu Wu}, \bibinfo{person}{Jiehui Xu}, \bibinfo{person}{Jianmin Wang}, {and} \bibinfo{person}{Mingsheng Long}.} \bibinfo{year}{2021}\natexlab{}.
\newblock \showarticletitle{Autoformer: Decomposition Transformers with Auto-Correlation for Long-Term Series Forecasting}. In \bibinfo{booktitle}{\emph{Advances in Neural Information Processing Systems 34: Annual Conference on Neural Information Processing Systems 2021, NeurIPS 2021, December 6-14, 2021, virtual}}, \bibfield{editor}{\bibinfo{person}{Marc'Aurelio Ranzato}, \bibinfo{person}{Alina Beygelzimer}, \bibinfo{person}{Yann~N. Dauphin}, \bibinfo{person}{Percy Liang}, {and} \bibinfo{person}{Jennifer~Wortman Vaughan}} (Eds.). \bibinfo{pages}{22419--22430}.
\newblock
\urldef\tempurl%
\url{https://proceedings.neurips.cc/paper/2021/hash/bcc0d400288793e8bdcd7c19a8ac0c2b-Abstract.html}
\showURL{%
\tempurl}


\bibitem[Xu et~al\mbox{.}(2022)]%
        {Xu22}
\bibfield{author}{\bibinfo{person}{Jiehui Xu}, \bibinfo{person}{Haixu Wu}, \bibinfo{person}{Jianmin Wang}, {and} \bibinfo{person}{Mingsheng Long}.} \bibinfo{year}{2022}\natexlab{}.
\newblock \showarticletitle{Anomaly Transformer: Time Series Anomaly Detection with Association Discrepancy}. In \bibinfo{booktitle}{\emph{The Tenth International Conference on Learning Representations, {ICLR} 2022, Virtual Event, April 25-29, 2022}}. \bibinfo{publisher}{OpenReview.net}.
\newblock
\urldef\tempurl%
\url{https://openreview.net/forum?id=LzQQ89U1qm\_}
\showURL{%
\tempurl}


\bibitem[Xu et~al\mbox{.}(2024)]%
        {Xu24}
\bibfield{author}{\bibinfo{person}{Zhijian Xu}, \bibinfo{person}{Ailing Zeng}, {and} \bibinfo{person}{Qiang Xu}.} \bibinfo{year}{2024}\natexlab{}.
\newblock \showarticletitle{{FITS:} Modeling Time Series with 10k Parameters}. In \bibinfo{booktitle}{\emph{The Twelfth International Conference on Learning Representations, {ICLR} 2024, Vienna, Austria, May 7-11, 2024}}. \bibinfo{publisher}{OpenReview.net}.
\newblock
\urldef\tempurl%
\url{https://openreview.net/forum?id=bWcnvZ3qMb}
\showURL{%
\tempurl}


\bibitem[Xue et~al\mbox{.}(2019)]%
        {Xue19}
\bibfield{author}{\bibinfo{person}{Ning Xue}, \bibinfo{person}{Isaac Triguero}, \bibinfo{person}{Grazziela~P. Figueredo}, {and} \bibinfo{person}{Dario Landa{-}Silva}.} \bibinfo{year}{2019}\natexlab{}.
\newblock \showarticletitle{Evolving Deep CNN-LSTMs for Inventory Time Series Prediction}. In \bibinfo{booktitle}{\emph{{IEEE} Congress on Evolutionary Computation, {CEC} 2019, Wellington, New Zealand, June 10-13, 2019}}. \bibinfo{publisher}{{IEEE}}, \bibinfo{pages}{1517--1524}.
\newblock
\href{https://doi.org/10.1109/CEC.2019.8789957}{doi:\nolinkurl{10.1109/CEC.2019.8789957}}


\bibitem[Yang et~al\mbox{.}(2023)]%
        {Yang23}
\bibfield{author}{\bibinfo{person}{Yiyuan Yang}, \bibinfo{person}{Chaoli Zhang}, \bibinfo{person}{Tian Zhou}, \bibinfo{person}{Qingsong Wen}, {and} \bibinfo{person}{Liang Sun}.} \bibinfo{year}{2023}\natexlab{}.
\newblock \showarticletitle{DCdetector: Dual Attention Contrastive Representation Learning for Time Series Anomaly Detection}. In \bibinfo{booktitle}{\emph{Proceedings of the 29th {ACM} {SIGKDD} Conference on Knowledge Discovery and Data Mining, {KDD} 2023, Long Beach, CA, USA, August 6-10, 2023}}, \bibfield{editor}{\bibinfo{person}{Ambuj~K. Singh}, \bibinfo{person}{Yizhou Sun}, \bibinfo{person}{Leman Akoglu}, \bibinfo{person}{Dimitrios Gunopulos}, \bibinfo{person}{Xifeng Yan}, \bibinfo{person}{Ravi Kumar}, \bibinfo{person}{Fatma Ozcan}, {and} \bibinfo{person}{Jieping Ye}} (Eds.). \bibinfo{publisher}{{ACM}}, \bibinfo{pages}{3033--3045}.
\newblock
\href{https://doi.org/10.1145/3580305.3599295}{doi:\nolinkurl{10.1145/3580305.3599295}}


\bibitem[Yi et~al\mbox{.}(2023)]%
        {Yi23}
\bibfield{author}{\bibinfo{person}{Kun Yi}, \bibinfo{person}{Qi Zhang}, \bibinfo{person}{Wei Fan}, \bibinfo{person}{Shoujin Wang}, \bibinfo{person}{Pengyang Wang}, \bibinfo{person}{Hui He}, \bibinfo{person}{Ning An}, \bibinfo{person}{Defu Lian}, \bibinfo{person}{Longbing Cao}, {and} \bibinfo{person}{Zhendong Niu}.} \bibinfo{year}{2023}\natexlab{}.
\newblock \showarticletitle{Frequency-domain MLPs are More Effective Learners in Time Series Forecasting}. In \bibinfo{booktitle}{\emph{Advances in Neural Information Processing Systems 36: Annual Conference on Neural Information Processing Systems 2023, NeurIPS 2023, New Orleans, LA, USA, December 10 - 16, 2023}}, \bibfield{editor}{\bibinfo{person}{Alice Oh}, \bibinfo{person}{Tristan Naumann}, \bibinfo{person}{Amir Globerson}, \bibinfo{person}{Kate Saenko}, \bibinfo{person}{Moritz Hardt}, {and} \bibinfo{person}{Sergey Levine}} (Eds.).
\newblock
\urldef\tempurl%
\url{http://papers.nips.cc/paper\_files/paper/2023/hash/f1d16af76939f476b5f040fd1398c0a3-Abstract-Conference.html}
\showURL{%
\tempurl}


\bibitem[Zhang et~al\mbox{.}(2024)]%
        {Zhang24}
\bibfield{author}{\bibinfo{person}{Yunhao Zhang}, \bibinfo{person}{Minghao Liu}, \bibinfo{person}{Shengyang Zhou}, {and} \bibinfo{person}{Junchi Yan}.} \bibinfo{year}{2024}\natexlab{}.
\newblock \showarticletitle{{UP2ME:} Univariate Pre-training to Multivariate Fine-tuning as a General-purpose Framework for Multivariate Time Series Analysis}. In \bibinfo{booktitle}{\emph{Forty-first International Conference on Machine Learning, {ICML} 2024, Vienna, Austria, July 21-27, 2024}}. \bibinfo{publisher}{OpenReview.net}.
\newblock
\urldef\tempurl%
\url{https://openreview.net/forum?id=aR3uxWlZhX}
\showURL{%
\tempurl}


\bibitem[Zhou et~al\mbox{.}(2021)]%
        {Zhou21}
\bibfield{author}{\bibinfo{person}{Haoyi Zhou}, \bibinfo{person}{Shanghang Zhang}, \bibinfo{person}{Jieqi Peng}, \bibinfo{person}{Shuai Zhang}, \bibinfo{person}{Jianxin Li}, \bibinfo{person}{Hui Xiong}, {and} \bibinfo{person}{Wancai Zhang}.} \bibinfo{year}{2021}\natexlab{}.
\newblock \showarticletitle{Informer: Beyond Efficient Transformer for Long Sequence Time-Series Forecasting}. In \bibinfo{booktitle}{\emph{Thirty-Fifth {AAAI} Conference on Artificial Intelligence, {AAAI} 2021, Thirty-Third Conference on Innovative Applications of Artificial Intelligence, {IAAI} 2021, The Eleventh Symposium on Educational Advances in Artificial Intelligence, {EAAI} 2021, Virtual Event, February 2-9, 2021}}. \bibinfo{publisher}{{AAAI} Press}, \bibinfo{pages}{11106--11115}.
\newblock
\href{https://doi.org/10.1609/AAAI.V35I12.17325}{doi:\nolinkurl{10.1609/AAAI.V35I12.17325}}


\bibitem[Zhou et~al\mbox{.}(2024)]%
        {zhou24}
\bibfield{author}{\bibinfo{person}{Peng Zhou}, \bibinfo{person}{Yufeng Guo}, \bibinfo{person}{Haoran Yu}, \bibinfo{person}{Yuanting Yan}, \bibinfo{person}{Yanping Zhang}, {and} \bibinfo{person}{Xindong Wu}.} \bibinfo{year}{2024}\natexlab{}.
\newblock \showarticletitle{Concept Evolution Detecting over Feature Streams}.
\newblock \bibinfo{journal}{\emph{ACM Transactions on Knowledge Discovery from Data}} \bibinfo{volume}{18}, \bibinfo{number}{8} (\bibinfo{year}{2024}), \bibinfo{pages}{1--32}.
\newblock


\bibitem[Zhou et~al\mbox{.}(2022)]%
        {Zhou22}
\bibfield{author}{\bibinfo{person}{Tian Zhou}, \bibinfo{person}{Ziqing Ma}, \bibinfo{person}{Qingsong Wen}, \bibinfo{person}{Xue Wang}, \bibinfo{person}{Liang Sun}, {and} \bibinfo{person}{Rong Jin}.} \bibinfo{year}{2022}\natexlab{}.
\newblock \showarticletitle{FEDformer: Frequency Enhanced Decomposed Transformer for Long-term Series Forecasting}. In \bibinfo{booktitle}{\emph{International Conference on Machine Learning, {ICML} 2022, 17-23 July 2022, Baltimore, Maryland, {USA}}} \emph{(\bibinfo{series}{Proceedings of Machine Learning Research}, Vol.~\bibinfo{volume}{162})}, \bibfield{editor}{\bibinfo{person}{Kamalika Chaudhuri}, \bibinfo{person}{Stefanie Jegelka}, \bibinfo{person}{Le~Song}, \bibinfo{person}{Csaba Szepesv{\'{a}}ri}, \bibinfo{person}{Gang Niu}, {and} \bibinfo{person}{Sivan Sabato}} (Eds.). \bibinfo{publisher}{{PMLR}}, \bibinfo{pages}{27268--27286}.
\newblock
\urldef\tempurl%
\url{https://proceedings.mlr.press/v162/zhou22g.html}
\showURL{%
\tempurl}


\end{thebibliography}
%%% -*-BibTeX-*-
%%% Do NOT edit. File created by BibTeX with style
%%% ACM-Reference-Format-Journals [18-Jan-2012].

%%
%% If your work has an appendix, this is the place to put it.
\appendix
\clearpage
\appendix

\section{Comparison Between FFT and Wavelet Transform}
\label{appendix: comparison}
Both the Fast Fourier Transform (FFT) and Wavelet Transform are fundamental tools for decomposing time series data from the time domain into the frequency domain, enabling lossless forward and inverse transformations. Despite this shared capability, they differ significantly in their underlying principles and suitability for general-purpose time series forecasting. In this section, we compare their characteristics and explain why FFT is generally more appropriate for modeling diverse time series data in forecasting tasks.

\begin{itemize}
    \item \textbf{Parameter-free nature of FFT:} FFT is a deterministic, parameter-free transformation that decomposes a signal into a fixed set of orthogonal frequency bases. The absence of hyperparameters eliminates the need for domain-specific knowledge or manual tuning during decomposition, making FFT highly suitable for modeling diverse time series in a general and automated manner.
    
    \item \textbf{Hyperparameter dependence of Wavelet Transform:} In contrast, the Wavelet Transform requires selecting specific wavelet functions and scale parameters, which act as hyperparameters critically influencing the decomposition results. Careful tuning of these parameters can enhance representation of time series exhibiting localized, transient, or non-stationary behaviors. However, this reliance on domain expertise and parameter selection limits its applicability in universal forecasting frameworks across heterogeneous datasets.
\end{itemize}

In summary, although both FFT and Wavelet Transform provide lossless time-frequency analysis, FFT’s parameter-free and universal nature makes it more suitable for general-purpose time series forecasting. Conversely, the Wavelet Transform is better suited to specialized scenarios where domain knowledge guides hyperparameter tuning to effectively capture complex localized features.

\section{Experimental Details}
\subsection{Implementation details}
We use a fixed look-back window (context length) of 96 time points to model the historical data and forecast future horizons $ T \in \{96, 192, 336, 720\} $. The patch length is varied between 8 and 48 to balance the trade-off between temporal resolution and computational efficiency.
Training is performed using mini-batch gradient descent with batch sizes ranging from 32 to 256. Larger batch sizes improve parallelism and enable more efficient utilization of GPU resources.
We adopt the ADAM optimizer for model optimization, tuning the learning rate over the set $\{1 \times 10^{-2}, 5 \times 10^{-3}, 2 \times 10^{-3}, 1 \times 10^{-3}, 5 \times 10^{-4}, 1 \times 10^{-4}\}$ to achieve stable and efficient convergence.
Early stopping is employed based on the validation loss: if the validation loss does not decrease for 8 consecutive epochs, which helps prevent overfitting and reduces unnecessary computation.
The model is implemented in PyTorch and trained on a single NVIDIA A100 GPU with 40GB of memory.
Evaluation metrics include Mean Squared Error (MSE) and Mean Absolute Error (MAE). We compare our results against the best-performing state-of-the-art models reported in the literature or reproduced from their published source codes.

\noindent\textbf{Summary of key hyperparameters:}
\begin{itemize}
    \item \textbf{Input length (look-back window):} 96
    \item \textbf{Forecast horizons:} \(T \in \{96, 192, 336, 720\}\)
    \item \textbf{Patch length:} 8 to 48
    \item \textbf{Batch size:} 32 to 256
    \item \textbf{Learning rates tested:} \(1 \times 10^{-2}, 5 \times 10^{-3}, 2 \times 10^{-3}, 1 \times 10^{-3}, 5 \times 10^{-4}, 1 \times 10^{-4}\)
    \item \textbf{Optimizer:} ADAM
    \item \textbf{Early stopping:} validation loss no improvement for 8 consecutive epochs
    \item \textbf{Hardware:} NVIDIA A100 GPU with 40GB memory
\end{itemize}

\subsection{Full results}
\label{appendix: full}
\para{Ablation results.} 
To complement the average results reported earlier, Table~\ref{table:module_ablation} presents the full forecasting performance for module ablation. \sysname\ achieves the best MSE in 27 out of 35 experiments and the best MAE in 30 out of 35. In contrast, \sysname\_{\text{advanced}} ranks second with only 6 and 3 best results on MSE and MAE, respectively. These results demonstrate the critical importance of modeling both data drift and basis evolution for improved forecasting accuracy.

\begin{table}[tbp]
\centering
\caption{Full results of module ablation}
\label{table:module_ablation}
\small % 调整字体大小 small/scriptsize/footnotesize
\setlength{\tabcolsep}{2.5pt} % 调整列间距
\renewcommand{\arraystretch}{0.8} % 调整行间距
\begin{tabular}{c|c|cc|cc|cc}
\toprule
model & &\multicolumn{2}{c|}{\sysname} & \multicolumn{2}{c|}{\sysname\_adv.} & \multicolumn{2}{c}{\sysname\_base} \\ \midrule
dataset & T &\multicolumn{1}{c|}{MSE} & MAE & \multicolumn{1}{c|}{MSE} & MAE & \multicolumn{1}{c|}{MSE} & MAE \\ \midrule
 & 96 &\multicolumn{1}{c|}{\textbf{0.365}} & \textbf{0.390} & \multicolumn{1}{c|}{0.369} & 0.391 & \multicolumn{1}{c|}{0.373} & 0.391 \\
 & 192 &\multicolumn{1}{c|}{\textbf{0.420}} & \textbf{0.418} & \multicolumn{1}{c|}{0.425} & 0.425 & \multicolumn{1}{c|}{0.422} & 0.421 \\
ETTh1 & 336 &\multicolumn{1}{c|}{\textbf{0.458}} & \textbf{0.437} & \multicolumn{1}{c|}{0.459} & 0.438 & \multicolumn{1}{c|}{0.466} & 0.438 \\
 & 720 &\multicolumn{1}{c|}{\textbf{0.456}} & \textbf{0.454} & \multicolumn{1}{c|}{0.470} & 0.467 & \multicolumn{1}{c|}{0.474} & 0.459 \\
&Avg. & \multicolumn{1}{c|}{\textbf{0.425}} & \textbf{0.425} & \multicolumn{1}{c|}{0.431} & 0.430 & \multicolumn{1}{c|}{0.434} & 0.427 \\ \midrule
 & 96 &\multicolumn{1}{c|}{0.282} & 0.333 & \multicolumn{1}{c|}{\textbf{0.280}} & \textbf{0.331} & \multicolumn{1}{c|}{0.281} & 0.333 \\
 & 192 &\multicolumn{1}{c|}{0.362} & 0.383 & \multicolumn{1}{c|}{\textbf{0.359}} & \textbf{0.382} & \multicolumn{1}{c|}{0.360} & 0.384 \\
ETTh2 & 336 &\multicolumn{1}{c|}{0.403} & \textbf{0.419} & \multicolumn{1}{c|}{\textbf{0.398}} & 0.421 & \multicolumn{1}{c|}{0.400} & 0.420 \\
 & 720 &\multicolumn{1}{c|}{\textbf{0.408}} & \textbf{0.433} & \multicolumn{1}{c|}{0.412} & 0.433 & \multicolumn{1}{c|}{0.409} & 0.433 \\
& Avg. & \multicolumn{1}{c|}{0.364} & \textbf{0.392} & \multicolumn{1}{c|}{\textbf{0.362}} & 0.392 & \multicolumn{1}{c|}{0.363} & 0.393 \\ \midrule
 & 96 &\multicolumn{1}{c|}{\textbf{0.310}} & \textbf{0.344} & \multicolumn{1}{c|}{0.313} & 0.353 & \multicolumn{1}{c|}{0.317} & 0.352 \\
 & 192 & \multicolumn{1}{c|}{\textbf{0.356}} & \textbf{0.375} & \multicolumn{1}{c|}{0.358} & 0.379 & \multicolumn{1}{c|}{0.360} & 0.378 \\
ETTm1 & 336 & \multicolumn{1}{c|}{\textbf{0.385}} & 0.397 & \multicolumn{1}{c|}{0.386} & 0.399 & \multicolumn{1}{c|}{0.383} & \textbf{0.397} \\
 & 720 & \multicolumn{1}{c|}{0.448} & \textbf{0.431} & \multicolumn{1}{c|}{0.444} & 0.434 & \multicolumn{1}{c|}{\textbf{0.443}} & 0.433 \\
&Avg. & \multicolumn{1}{c|}{\textbf{0.375}} & \textbf{0.387} & \multicolumn{1}{c|}{0.375} & 0.391 & \multicolumn{1}{c|}{0.376} & 0.390 \\ \midrule
 & 96 &\multicolumn{1}{c|}{\textbf{0.170}} & \textbf{0.252} & \multicolumn{1}{c|}{0.172} & 0.253 & \multicolumn{1}{c|}{0.172} & 0.253 \\
 & 192 & \multicolumn{1}{c|}{\textbf{0.237}} & \textbf{0.297} & \multicolumn{1}{c|}{0.238} & 0.298 & \multicolumn{1}{c|}{0.237} & 0.297 \\
ETTm2 & 336 & \multicolumn{1}{c|}{0.299} & 0.338 & \multicolumn{1}{c|}{0.300} & \textbf{0.338} & \multicolumn{1}{c|}{\textbf{0.296}} & 0.334 \\
 & 720 & \multicolumn{1}{c|}{0.399} & 0.395 & \multicolumn{1}{c|}{\textbf{0.398}} & 0.397 & \multicolumn{1}{c|}{0.394} & \textbf{0.394} \\
&Avg. & \multicolumn{1}{c|}{0.276} & 0.321 & \multicolumn{1}{c|}{0.277} & 0.322 & \multicolumn{1}{c|}{\textbf{0.275}} & \textbf{0.320} \\ \midrule
 & 96 &\multicolumn{1}{c|}{\textbf{0.162}} & \textbf{0.204} & \multicolumn{1}{c|}{0.164} & 0.206 & \multicolumn{1}{c|}{0.165} & 0.208 \\
 & 192 & \multicolumn{1}{c|}{\textbf{0.207}} & \textbf{0.246} & \multicolumn{1}{c|}{0.209} & 0.247 & \multicolumn{1}{c|}{0.209} & 0.248 \\
Weather & 336& \multicolumn{1}{c|}{\textbf{0.263}} & \textbf{0.287} & \multicolumn{1}{c|}{0.266} & 0.293 & \multicolumn{1}{c|}{0.267} & 0.291 \\
 & 720& \multicolumn{1}{c|}{\textbf{0.340}} & \textbf{0.338} & \multicolumn{1}{c|}{0.342} & 0.341 & \multicolumn{1}{c|}{0.343} & 0.340 \\
&Avg. & \multicolumn{1}{c|}{\textbf{0.243}} & \textbf{0.269} & \multicolumn{1}{c|}{0.245} & 0.272 & \multicolumn{1}{c|}{0.246} & 0.272 \\ \midrule
 & 96 &\multicolumn{1}{c|}{\textbf{0.474}} & \textbf{0.272} & \multicolumn{1}{c|}{0.479} & 0.285 & \multicolumn{1}{c|}{0.493} & 0.304 \\
 & 192 & \multicolumn{1}{c|}{0.487} & \textbf{0.269} & \multicolumn{1}{c|}{\textbf{0.480}} & 0.286 & \multicolumn{1}{c|}{0.486} & 0.299 \\
Traffic & 336 & \multicolumn{1}{c|}{\textbf{0.484}} & \textbf{0.275} & \multicolumn{1}{c|}{0.490} & 0.281 & \multicolumn{1}{c|}{0.507} & 0.306 \\
 & 720 & \multicolumn{1}{c|}{\textbf{0.531}} & \textbf{0.295} & \multicolumn{1}{c|}{0.532} & 0.309 & \multicolumn{1}{c|}{0.536} & 0.324 \\
&Avg. & \multicolumn{1}{c|}{\textbf{0.494}} & \textbf{0.278} & \multicolumn{1}{c|}{0.495} & 0.290 & \multicolumn{1}{c|}{0.506} & 0.308 \\ \midrule
& 96 &\multicolumn{1}{c|}{\textbf{0.148}} & \textbf{0.236} & \multicolumn{1}{c|}{0.151} & 0.239 & \multicolumn{1}{c|}{0.162} & 0.249 \\
 & 192 &\multicolumn{1}{c|}{\textbf{0.161}} & \textbf{0.249} & \multicolumn{1}{c|}{0.163} & 0.250 & \multicolumn{1}{c|}{0.172} & 0.258 \\
Elc  & 336 & \multicolumn{1}{c|}{\textbf{0.176}} & \textbf{0.265} & \multicolumn{1}{c|}{0.179} & 0.267 & \multicolumn{1}{c|}{0.189} & 0.275 \\
 & 720 & \multicolumn{1}{c|}{\textbf{0.215}} & \textbf{0.299} & \multicolumn{1}{c|}{0.217} & 0.301 & \multicolumn{1}{c|}{0.232} & 0.310 \\
& Avg. & \multicolumn{1}{c|}{\textbf{0.175}} & \textbf{0.262} & \multicolumn{1}{c|}{0.178} & 0.264 & \multicolumn{1}{c|}{0.189} & 0.273 \\ \midrule
Best\_Count & &\multicolumn{1}{c|}{\textbf{27/35}} & \textbf{30/35} & \multicolumn{1}{c|}{6/35} & 3/35 & \multicolumn{1}{c|}{3/35} & 2/35 \\ \bottomrule
\end{tabular}
\end{table}

Additionally, Table~\ref{table:loss_ablation} presents the full results of loss function ablations across seven datasets. Our full model, \sysname, achieves the best MSE and MAE in 20 and 22 out of all experiments, respectively, outperforming all variants. The systematic performance degradation observed when removing each loss component confirms the essential contribution of every loss term to the overall forecasting accuracy. This ablation study validates the design of the composite loss in enhancing model effectiveness.

\begin{table}[!t]
\centering
\caption{Full results of Loss ablation}
\label{table:loss_ablation}
\small % 调整字体大小 small/scriptsize/footnotesize
\setlength{\tabcolsep}{2pt} % 调整列间距
\renewcommand{\arraystretch}{1} % 调整行间距
\begin{tabular}{c|l|cc|cc|cc|cc}
\toprule
Model &  & \multicolumn{2}{c|}{\textbackslash{}sysname} & \multicolumn{2}{c|}{\textbackslash{}sysname\_enh.} & \multicolumn{2}{c|}{\textbackslash{}sysname\_adv.} & \multicolumn{2}{c}{\textbackslash{}sysname\_base} \\ \midrule
Data & \multicolumn{1}{c|}{T} & \multicolumn{1}{c|}{MSE} & MAE & \multicolumn{1}{c|}{MSE} & MAE & \multicolumn{1}{c|}{MSE} & MAE & \multicolumn{1}{c|}{MSE} & MAE \\ \midrule
 & 96 & \multicolumn{1}{c|}{\textbf{0.365}} & \textbf{0.390} & \multicolumn{1}{c|}{0.369} & 0.391 & \multicolumn{1}{c|}{0.385} & 0.401 & \multicolumn{1}{c|}{0.381} & 0.400 \\
 & 192 & \multicolumn{1}{c|}{\textbf{0.420}} & \textbf{0.418} & \multicolumn{1}{c|}{0.423} & 0.420 & \multicolumn{1}{c|}{0.432} & 0.434 & \multicolumn{1}{c|}{0.426} & 0.429 \\
ETTh1 & 336 & \multicolumn{1}{c|}{\textbf{0.458}} & \textbf{0.437} & \multicolumn{1}{c|}{0.462} & 0.438 & \multicolumn{1}{c|}{0.465} & 0.447 & \multicolumn{1}{c|}{0.458} & 0.441 \\
 & 720 & \multicolumn{1}{c|}{\textbf{0.456}} & \textbf{0.454} & \multicolumn{1}{c|}{0.461} & 0.459 & \multicolumn{1}{c|}{0.474} & 0.467 & \multicolumn{1}{c|}{0.468} & 0.464 \\
 & \multicolumn{1}{c|}{Avg.} & \multicolumn{1}{c|}{\textbf{0.424}} & \textbf{0.424} & \multicolumn{1}{c|}{0.428} & 0.427 & \multicolumn{1}{c|}{0.439} & 0.437 & \multicolumn{1}{c|}{0.433} & 0.433 \\ \midrule
 & 96 & \multicolumn{1}{c|}{0.282} & \textbf{0.333} & \multicolumn{1}{c|}{\textbf{0.281}} & 0.333 & \multicolumn{1}{c|}{0.300} & 0.348 & \multicolumn{1}{c|}{0.283} & 0.332 \\
 & 192 & \multicolumn{1}{c|}{0.362} & 0.383 & \multicolumn{1}{c|}{\textbf{0.360}} & \textbf{0.382} & \multicolumn{1}{c|}{0.382} & 0.397 & \multicolumn{1}{c|}{0.360} & 0.384 \\
ETTh2 & 336 & \multicolumn{1}{c|}{\textbf{0.403}} & 0.419 & \multicolumn{1}{c|}{0.403} & \textbf{0.418} & \multicolumn{1}{c|}{0.425} & 0.432 & \multicolumn{1}{c|}{0.406} & 0.425 \\
 & 720 & \multicolumn{1}{c|}{\textbf{0.408}} & \textbf{0.433} & \multicolumn{1}{c|}{0.411} & 0.433 & \multicolumn{1}{c|}{0.434} & 0.448 & \multicolumn{1}{c|}{0.420} & 0.438 \\
 & \multicolumn{1}{c|}{Avg.} & \multicolumn{1}{c|}{\textbf{0.363}} & 0.392 & \multicolumn{1}{c|}{0.363} & \textbf{0.391} & \multicolumn{1}{c|}{0.385} & 0.406 & \multicolumn{1}{c|}{0.367} & 0.394 \\ \midrule
 & 96 & \multicolumn{1}{c|}{\textbf{0.310}} & \textbf{0.344} & \multicolumn{1}{c|}{0.312} & 0.349 & \multicolumn{1}{c|}{0.325} & 0.368 & \multicolumn{1}{c|}{0.316} & 0.357 \\
 & 192 & \multicolumn{1}{c|}{\textbf{0.356}} & \textbf{0.375} & \multicolumn{1}{c|}{0.359} & 0.376 & \multicolumn{1}{c|}{0.369} & 0.390 & \multicolumn{1}{c|}{0.360} & 0.380 \\
ETTm1 & 336 & \multicolumn{1}{c|}{0.385} & 0.397 & \multicolumn{1}{c|}{\textbf{0.383}} & \textbf{0.395} & \multicolumn{1}{c|}{0.390} & 0.407 & \multicolumn{1}{c|}{0.389} & 0.403 \\
 & 720 & \multicolumn{1}{c|}{0.448} & 0.431 & \multicolumn{1}{c|}{\textbf{0.442}} & \textbf{0.430} & \multicolumn{1}{c|}{0.453} & 0.442 & \multicolumn{1}{c|}{0.448} & 0.441 \\
 & \multicolumn{1}{c|}{Avg.} & \multicolumn{1}{c|}{0.374} & \textbf{0.386} & \multicolumn{1}{c|}{\textbf{0.374}} & 0.387 & \multicolumn{1}{c|}{0.384} & 0.401 & \multicolumn{1}{c|}{0.378} & 0.395 \\ \midrule
 & 96 & \multicolumn{1}{c|}{\textbf{0.170}} & 0.252 & \multicolumn{1}{c|}{0.170} & \textbf{0.250} & \multicolumn{1}{c|}{0.192} & 0.283 & \multicolumn{1}{c|}{0.177} & 0.262 \\
 & 192 & \multicolumn{1}{c|}{\textbf{0.237}} & \textbf{0.297} & \multicolumn{1}{c|}{0.238} & 0.297 & \multicolumn{1}{c|}{0.254} & 0.317 & \multicolumn{1}{c|}{0.242} & 0.303 \\
ETTm2 & 336 & \multicolumn{1}{c|}{\textbf{0.299}} & 0.338 & \multicolumn{1}{c|}{0.300} & \textbf{0.337} & \multicolumn{1}{c|}{0.319} & 0.358 & \multicolumn{1}{c|}{0.305} & 0.344 \\
 & 720 & \multicolumn{1}{c|}{\textbf{0.399}} & 0.395 & \multicolumn{1}{c|}{0.399} & \textbf{0.394} & \multicolumn{1}{c|}{0.419} & 0.416 & \multicolumn{1}{c|}{0.404} & 0.401 \\
 & \multicolumn{1}{c|}{Avg.} & \multicolumn{1}{c|}{\textbf{0.276}} & 0.320 & \multicolumn{1}{c|}{0.277} & \textbf{0.319} & \multicolumn{1}{c|}{0.296} & 0.343 & \multicolumn{1}{c|}{0.282} & 0.327 \\ \midrule
 & 96 & \multicolumn{1}{c|}{0.162} & 0.204 & \multicolumn{1}{c|}{\textbf{0.160}} & \textbf{0.202} & \multicolumn{1}{c|}{0.163} & 0.207 & \multicolumn{1}{c|}{0.162} & 0.205 \\
 & 192 & \multicolumn{1}{c|}{0.207} & 0.246 & \multicolumn{1}{c|}{\textbf{0.206}} & \textbf{0.244} & \multicolumn{1}{c|}{0.208} & 0.249 & \multicolumn{1}{c|}{0.207} & 0.247 \\
Weather & 336 & \multicolumn{1}{c|}{\textbf{0.263}} & \textbf{0.287} & \multicolumn{1}{c|}{0.264} & 0.287 & \multicolumn{1}{c|}{0.264} & 0.288 & \multicolumn{1}{c|}{0.267} & 0.291 \\
 & 720 & \multicolumn{1}{c|}{\textbf{0.340}} & \textbf{0.338} & \multicolumn{1}{c|}{0.342} & 0.338 & \multicolumn{1}{c|}{0.344} & 0.34 & \multicolumn{1}{c|}{0.344} & 0.339 \\
 & \multicolumn{1}{c|}{Avg.} & \multicolumn{1}{c|}{\textbf{0.243}} & 0.268 & \multicolumn{1}{c|}{0.243} & \textbf{0.267} & \multicolumn{1}{c|}{0.2448} & 0.2710 & \multicolumn{1}{c|}{0.2450} & 0.2705 \\ \midrule
 & 96 & \multicolumn{1}{c|}{0.474} & \textbf{0.272} & \multicolumn{1}{c|}{\textbf{0.466}} & 0.284 & \multicolumn{1}{c|}{0.481} & 0.270 & \multicolumn{1}{c|}{0.481} & 0.278 \\
 & 192 & \multicolumn{1}{c|}{0.487} & \textbf{0.269} & \multicolumn{1}{c|}{\textbf{0.475}} & 0.287 & \multicolumn{1}{c|}{0.499} & 0.282 & \multicolumn{1}{c|}{0.492} & 0.283 \\
Traffic & 336 & \multicolumn{1}{c|}{0.484} & \textbf{0.275} & \multicolumn{1}{c|}{\textbf{0.482}} & 0.278 & \multicolumn{1}{c|}{0.509} & 0.289 & \multicolumn{1}{c|}{0.516} & 0.293 \\
 & 720 & \multicolumn{1}{c|}{0.531} & \textbf{0.295} & \multicolumn{1}{c|}{\textbf{0.527}} & 0.297 & \multicolumn{1}{c|}{0.547} & 0.307 & \multicolumn{1}{c|}{0.552} & 0.306 \\
 & \multicolumn{1}{c|}{Avg.} & \multicolumn{1}{c|}{0.494} & \textbf{0.277} & \multicolumn{1}{c|}{\textbf{0.487}} & 0.286 & \multicolumn{1}{c|}{0.509} & 0.287 & \multicolumn{1}{c|}{0.510} & 0.290 \\ \midrule
 & 96 & \multicolumn{1}{c|}{\textbf{0.148}} & \textbf{0.236} & \multicolumn{1}{c|}{0.149} & 0.238 & \multicolumn{1}{c|}{0.153} & 0.244 & \multicolumn{1}{c|}{0.154} & 0.243 \\
 & 192 & \multicolumn{1}{c|}{0.161} & 0.249 & \multicolumn{1}{c|}{\textbf{0.160}} & \textbf{0.248} & \multicolumn{1}{c|}{0.167} & 0.256 & \multicolumn{1}{c|}{0.166} & 0.256 \\
Elc & 336 & \multicolumn{1}{c|}{\textbf{0.176}} & \textbf{0.265} & \multicolumn{1}{c|}{0.176} & 0.266 & \multicolumn{1}{c|}{0.181} & 0.273 & \multicolumn{1}{c|}{0.182} & 0.274 \\
 & 720 & \multicolumn{1}{c|}{0.215} & \textbf{0.299} & \multicolumn{1}{c|}{\textbf{0.214}} & 0.299 & \multicolumn{1}{c|}{0.222} & 0.308 & \multicolumn{1}{c|}{0.225} & 0.309 \\
 & \multicolumn{1}{c|}{Avg.} & \multicolumn{1}{c|}{0.175} & \textbf{0.262} & \multicolumn{1}{c|}{\textbf{0.174}} & 0.262 & \multicolumn{1}{c|}{0.180} & 0.270 & \multicolumn{1}{c|}{0.181} & 0.270 \\ \midrule
Best &  & \multicolumn{1}{l|}{\textbf{20/35}} & \multicolumn{1}{l|}{\textbf{22/35}} & \multicolumn{1}{l|}{15/35} & 13/35 & \multicolumn{1}{c|}{0} & 0 & \multicolumn{1}{c|}{0} & 0 \\ \bottomrule
\end{tabular}
\end{table}

\para{Model sensitivity to look-back window length.}
We evaluate the impact of varying look-back window sizes \{96, 192, 288, 384, 512\} on forecasting performance using ETTh1 and Weather datasets, with other hyper-parameters fixed. 
As shown in Table~\ref{table:look_back_win_size} and Figure \ref{fig:look-back-win}, for ETTh1, increasing the window size from 96 to 384 consistently reduces the average MSE from 0.439 to 0.414 and MAE from 0.444 to 0.427, indicating improved accuracy due to more historical information. However, further increasing the window to 512 leads to a slight increase in MSE (0.439) and MAE (0.433), suggesting diminishing returns or potential noise introduction.
Similarly, on the Weather dataset, average MSE and MAE decrease steadily from 0.245 and 0.270 at window 96 to 0.224 and 0.261 at window 512, showing consistent gains with longer look-back windows. The improvements are less pronounced beyond window size 384, indicating performance saturation.

Overall, these results demonstrate that enlarging the look-back window generally enhances forecasting accuracy by leveraging more temporal context, but beyond a certain length, the benefits plateau or slightly decline, likely due to noise accumulation and redundancy in the input data.

\begin{table}[tbp]
\centering
\caption{Forecasting results of various look back window sizes}
\label{table:look_back_win_size}
\scriptsize % 调整字体大小 small/scriptsize/footnotesize
\setlength{\tabcolsep}{4pt} % 调整列间距
\renewcommand{\arraystretch}{0.8} % 调整行间距
\begin{tabular}{c|cc|cc|cc|cc|cc}
\toprule
Window size & \multicolumn{2}{c|}{96} & \multicolumn{2}{c|}{192} & \multicolumn{2}{c|}{288} & \multicolumn{2}{c|}{384} & \multicolumn{2}{c}{512} \\ \midrule
Dataset & \multicolumn{1}{c|}{MSE} & MAE & \multicolumn{1}{c|}{MSE} & MAE & \multicolumn{1}{c|}{MSE} & MAE & \multicolumn{1}{c|}{MSE} & MAE & \multicolumn{1}{c|}{MSE} & MAE \\
\multicolumn{1}{l|}{} & \multicolumn{1}{c|}{0.376} & 0.396 & \multicolumn{1}{c|}{\textbf{0.373}} & 0.396 & \multicolumn{1}{c|}{0.375} & \textbf{0.395} & \multicolumn{1}{c|}{0.376} & 0.398 & \multicolumn{1}{c|}{0.382} & 0.405 \\
 & \multicolumn{1}{c|}{0.433} & 0.466 & \multicolumn{1}{c|}{0.425} & 0.422 & \multicolumn{1}{c|}{0.417} & 0.418 & \multicolumn{1}{c|}{\textbf{0.418}} & 0.423 & \multicolumn{1}{c|}{\textbf{0.406}} & 0.421 \\
ETTh1 & \multicolumn{1}{c|}{0.466} & 0.442 & \multicolumn{1}{c|}{0.449} & 0.433 & \multicolumn{1}{c|}{0.440} & 0.434 & \multicolumn{1}{c|}{0.429} & \textbf{0.431} & \multicolumn{1}{c|}{\textbf{0.427}} & 0.442 \\
 & \multicolumn{1}{c|}{0.480} & 0.470 & \multicolumn{1}{c|}{0.509} & 0.489 & \multicolumn{1}{c|}{0.540} & 0.510 & \multicolumn{1}{c|}{\textbf{0.433}} & \textbf{0.458} & \multicolumn{1}{c|}{0.440} & 0.463 \\
Avg. & \multicolumn{1}{c|}{0.439} & 0.444 & \multicolumn{1}{c|}{0.439} & 0.435 & \multicolumn{1}{c|}{0.443} & 0.439 & \multicolumn{1}{c|}{0.414} & \textbf{0.427} & \multicolumn{1}{c|}{\textbf{0.414}} & 0.433 \\ \midrule
\multicolumn{1}{l|}{} & \multicolumn{1}{c|}{0.162} & 0.204 & \multicolumn{1}{c|}{0.154} & 0.195 & \multicolumn{1}{c|}{0.147} & 0.194 & \multicolumn{1}{c|}{0.150} & 0.195 & \multicolumn{1}{c|}{\textbf{0.146}} & \textbf{0.192} \\
 & \multicolumn{1}{c|}{0.209} & 0.247 & \multicolumn{1}{c|}{0.198} & 0.237 & \multicolumn{1}{c|}{0.194} & 0.239 & \multicolumn{1}{c|}{0.195} & 0.241 & \multicolumn{1}{c|}{\textbf{0.190}} & \textbf{0.236} \\
Weather & \multicolumn{1}{c|}{0.267} & 0.289 & \multicolumn{1}{c|}{0.252} & 0.282 & \multicolumn{1}{c|}{0.251} & 0.281 & \multicolumn{1}{c|}{\textbf{0.246}} & \textbf{0.281} & \multicolumn{1}{c|}{0.248} & 0.283 \\
 & \multicolumn{1}{c|}{0.343} & 0.339 & \multicolumn{1}{c|}{0.330} & 0.335 & \multicolumn{1}{c|}{0.322} & 0.332 & \multicolumn{1}{c|}{0.317} & \textbf{0.331} & \multicolumn{1}{c|}{\textbf{0.312}} & 0.333 \\ 
Avg. & \multicolumn{1}{c|}{0.245} & 0.270 & \multicolumn{1}{c|}{0.233} & 0.262 & \multicolumn{1}{c|}{0.229} & 0.261 & \multicolumn{1}{c|}{0.227} & 0.262 & \multicolumn{1}{c|}{\textbf{0.224}} & \textbf{0.261} \\ \bottomrule
\end{tabular}
\end{table}

\begin{figure}[tbp]
    \centering
    % 第一个子图，宽度设为文本宽度的一半左右，可根据需要微调
    \subfigure[ETTh1]{
        \includegraphics[width=0.225\textwidth]{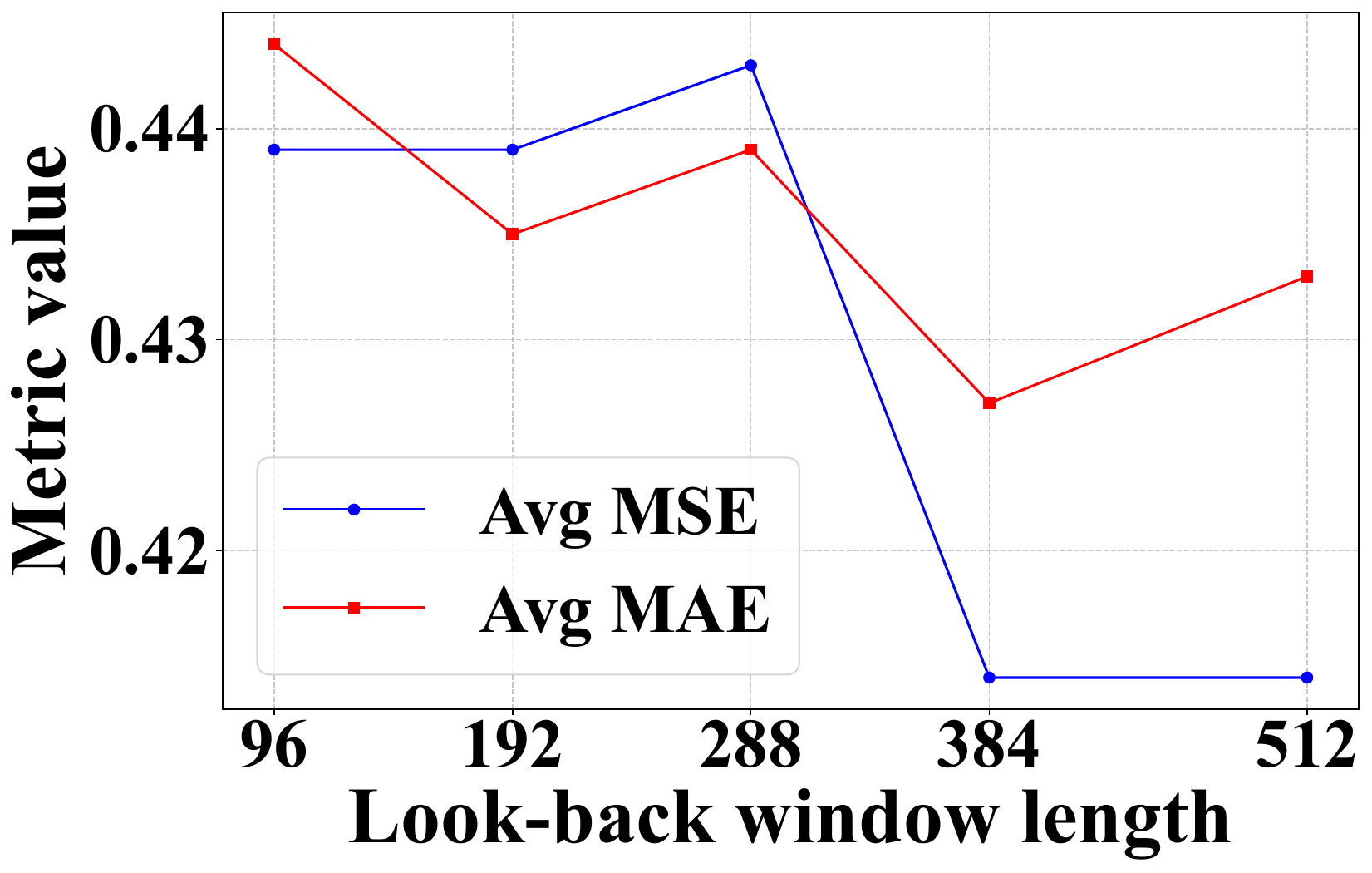}
    }
    % 用于在两个子图之间添加水平间距
    \hfill
    % 第二个子图，宽度同样设为文本宽度的一半左右
    \subfigure[Weather]{
        \includegraphics[width=0.225\textwidth]{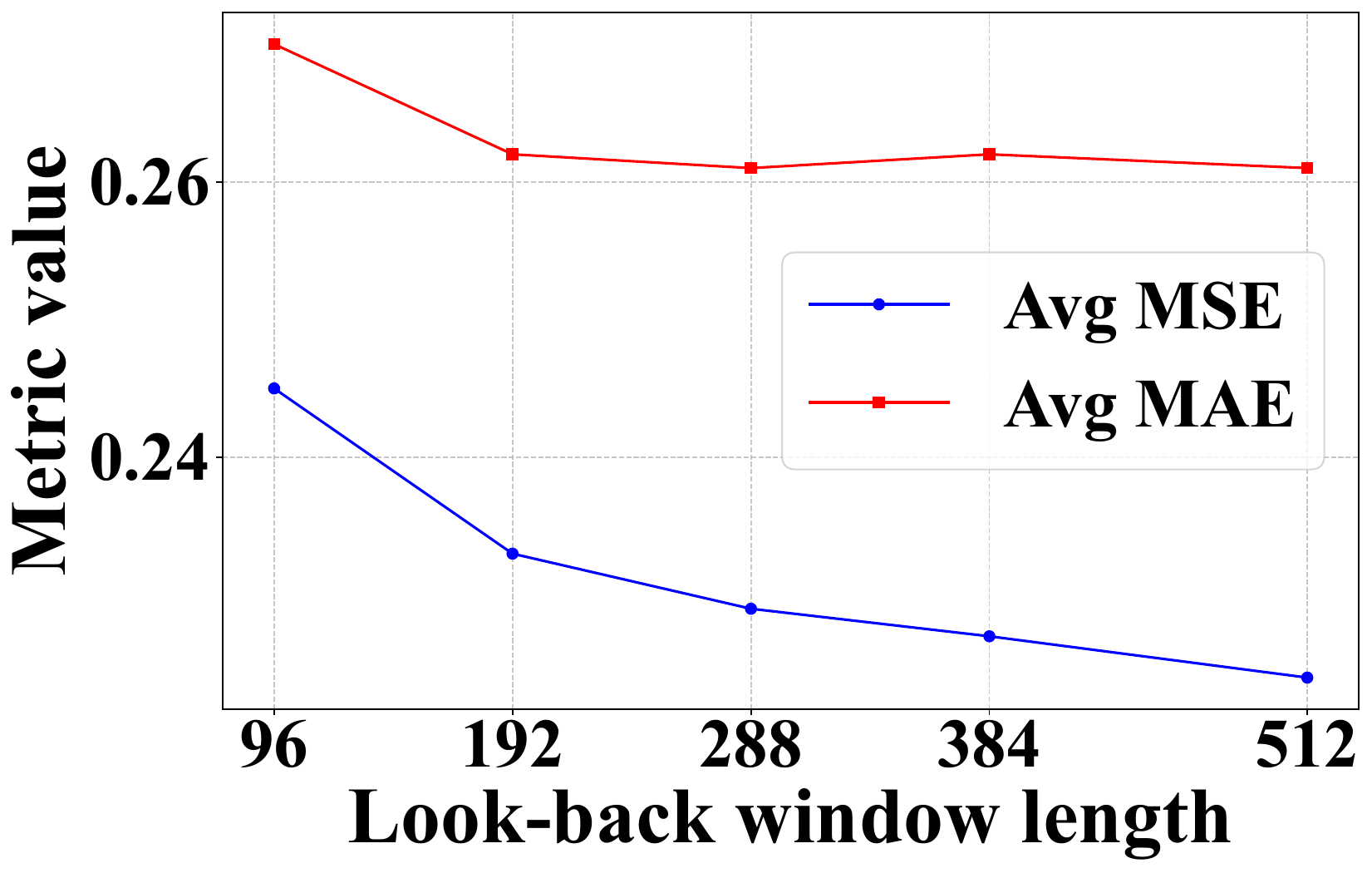}
    }
    \caption{Average forecasting results on ETTh1 and Weather datasets with various look back window lengths.}
    \label{fig:look-back-win}
\end{figure}

\end{document}